\pdfoutput=1

\newif\ifdraftmode
\draftmodefalse  % Change to \draftmodetrue for drafting

\ifdraftmode
  \documentclass[letterpaper, 11pt]{article}
\else
  \documentclass[letterpaper, 10pt, conference]{ieeeconf}
  \IEEEoverridecommandlockouts
\fi

\providecommand{\IEEEoverridecommandlockouts}{}
\providecommand{\overrideIEEEmargins}{}

\let\labelindent\relax

\usepackage{multicol}

\makeatletter
\let\NAT@parse\undefined
\makeatother

    \usepackage{amsthm}
	\usepackage{amsmath}
	\usepackage{verbatim} % multi-line comments
	\usepackage{amsfonts}
	\usepackage{amssymb}
	\usepackage{ragged2e}
	\usepackage{graphicx}
	\usepackage{cancel}
	\usepackage{mathtools}
	\usepackage{tabularx}
     \usepackage{adjustbox} % vertical alignment of figures
	\usepackage{arydshln}
	\usepackage{tensor}
	\usepackage{array}
	\usepackage[dvipsnames]{xcolor}
	\usepackage{listings}
	\usepackage{textcomp}
	\usepackage{mathrsfs}
	\usepackage{bbm}
	\usepackage{tikz}
	\usepackage{tikz-cd}
	\usepackage{enumitem}
	\usepackage{arydshln}
	\usepackage{relsize}
	\usepackage{multirow}
	\usepackage{scalerel}
	\usepackage{upgreek}
	\usepackage{ifthen}
	\usepackage{yhmath}
	\usepackage{blkarray}
	\usepackage{dashrule}
	\usepackage{subcaption}
	\usepackage{overpic}
    \usepackage[commandnameprefix=ifneeded]{changes}
    \usepackage[numbers]{natbib} 
    \usepackage{booktabs}

		\newcommand{\set}[1]{\left\{#1\right\}}
		
		\newcommand{\naturals}{\mathbb{N}}
		\newcommand{\N}{\naturals}

		\newcommand{\reals}{\mathbb{R}}
		\newcommand{\R}{\reals}

		\newcommand{\iprod}[1]{\left<#1\right>}

		\newcommand{\bmat}[1]{\begin{bmatrix}#1\end{bmatrix}}

		\DeclareMathOperator{\subto}{s.t.}

		\newtheorem{theorem}{Theorem}
        \newtheorem*{theorem*}{Theorem}

		\newtheorem{lemma}{Lemma}
		\newtheorem{corollary}{Corollary}

	\usepackage{letltxmacro}
	\LetLtxMacro\orgvdots\vdots
	\LetLtxMacro\orgddots\ddots

	\makeatletter
	\DeclareRobustCommand\vdots{%
		\mathpalette\@vdots{}%
	}
	\newcommand*{\@vdots}[2]{%
		\sbox0{$#1\cdotp\cdotp\cdotp\m@th$}%
		\sbox2{$#1.\m@th$}%
		\vbox{%
			\dimen@=\wd0 %
			\advance\dimen@ -3\ht2 %
			\kern.5\dimen@
			\dimen@=\wd2 %
			\advance\dimen@ -\ht2 %
			\dimen2=\wd0 %
			\advance\dimen2 -\dimen@
			\vbox to \dimen2{%
				\offinterlineskip
				\copy2 \vfill\copy2 \vfill\copy2 %
			}%
		}%
	}
	\DeclareRobustCommand\ddots{%
		\mathinner{%
			\mathpalette\@ddots{}%
			\mkern\thinmuskip
		}%
	}
	\newcommand*{\@ddots}[2]{%
		\sbox0{$#1\cdotp\cdotp\cdotp\m@th$}%
		\sbox2{$#1.\m@th$}%
		\vbox{%
			\dimen@=\wd0 %
			\advance\dimen@ -3\ht2 %
			\kern.5\dimen@
			\dimen@=\wd2 %
			\advance\dimen@ -\ht2 %
			\dimen2=\wd0 %
			\advance\dimen2 -\dimen@
			\vbox to \dimen2{%
				\offinterlineskip
				\hbox{$#1\mathpunct{.}\m@th$}%
				\vfill
				\hbox{$#1\mathpunct{\kern\wd2}\mathpunct{.}\m@th$}%
				\vfill
				\hbox{$#1\mathpunct{\kern\wd2}\mathpunct{\kern\wd2}\mathpunct{.}\m@th$}%
			}%
		}%
	}
	\makeatother

	\tikzset{
	  symbol/.style={
		draw=none,
		every to/.append style={
		  edge node={node [sloped, allow upside down, auto=false]{$#1$}}}
	  }
	}

\usepackage[bookmarks=true]{hyperref}

\usepackage{cleveref}
\Crefname{figure}{Figure}{Figures}
\Crefname{table}{Table}{Tables}
\Crefname{equation}{Eq.}{Eqs.}
\Crefname{section}{Section}{Sections}
\Crefname{subsection}{Subsection}{Subsections}
\Crefname{appendix}{Appendix}{Appendices}

\usepackage[commentColor=black]{algpseudocodex}
\tikzset{algpxIndentLine/.style={draw=black}}
\algrenewcommand{\alglinenumber}[1]{\bfseries\footnotesize #1}
\algrenewcommand{\textproc}{}
\algrenewcommand{\algorithmicrequire}{\textbf{Input:}}
\algrenewcommand{\algorithmicensure}{\textbf{Output:}}
\usepackage{algorithm}
\makeatletter
\newcommand{\algorithmname}{\ALG@name}
\renewcommand{\floatc@ruled}[2]{{\@fs@cfont #1:} #2\par}
\makeatother

\usepackage{dblfloatfix}

\usetikzlibrary{shapes.geometric}
\usetikzlibrary{angles, quotes, calc}
\usetikzlibrary{arrows.meta}

\usepackage[acronym]{glossaries}

\newacronym{gcs}{GCS}{Graph of Convex Sets}
\newacronym{qp}{QP}{Quadratic Program}
\newacronym{qcqp}{QCQP}{Quadratically Constrained Quadratic Program}
\newacronym{qcqps}{QCQPs}{Quadratically Constrained Quadratic Programs}
\newacronym{sdp}{SDP}{Semidefinite Program}
\newacronym{psd}{PSD}{Positive Semidefinite}
\newacronym{sdr}{SDR}{Semidefinite Relaxation}
\newacronym{sdrs}{SDRs}{Semidefinite Relaxations}
\newacronym{mip}{MIP}{Mixed-Integer Program}
\newacronym{micp}{MICP}{Mixed-Integer Convex Program}
\newacronym{spp}{SPP}{Shortest Path Problem}
\newacronym{com}{COM}{Center of Mass}
\newacronym{soc}{SOC}{Second-Order Cone}
\newacronym{mpc}{MPC}{Model-Predictive Control}
\newacronym{admm}{ADMM}{Alternating Direction Method of Multipliers}
\newacronym{lp}{LP}{Linear Program}
\newacronym{madmm}{mADMM}{Multiblock ADMM}
\newacronym{socp}{SOCP}{Second-Order Cone Program}
\newacronym{rsoc}{RSOC}{Rotated Second-Order Cone}
\newacronym{sos}{SOS}{Sum-of-Squares}
\newacronym{lmi}{LMI}{Linear Matrix Inequality}
\newacronym{pmi}{PMI}{Polynomial Matrix Inequality}
\newacronym{bvp}{BVP}{Boundary Value Problem}
\newacronym{nlp}{NLP}{Nonlinear Programming}

\usepackage{preamble/color-edits}

\addauthor{bpg}{cyan}
\addauthor{rt}{teal}
\addauthor{aa}{magenta}
\addauthor{pp}{red}
\addauthor{claude}{orange}

\title{Semidefinite Relaxations for \\ Collision-Free Motion Planning}

\author{
Bernhard Paus Graesdal$^1$,
Alexandre Amice$^1$,
Pablo A. Parrilo$^1$, and
Russ Tedrake$^1$% <- this % stops a space
\thanks{$^1$ Department of Electrical Engineering and Computer Science, Massachusetts Institute of Technology, Cambridge, MA, USA.
E-mail: \texttt{\{graesdal, amice, parrilo, russt\}@mit.edu}}
}

\begin{document}

\maketitle
\thispagestyle{empty}
\pagestyle{empty}

\begin{abstract}
We study semidefinite relaxations for collision-free motion planning.
We focus on a point robot moving from start to goal through spherical obstacles in $\R^n$, subject to path continuity constraints and squared derivative costs; a setting that is conceptually simple yet captures the hardness of collision-free motion planning.
We formulate this problem exactly as a nonconvex problem over polynomial curves, and present a natural semidefinite relaxation.
We contribute two key theoretical insights;
to our knowledge this is the first theoretical analysis of semidefinite relaxations for collision-free motion planning.
First, we show that solving the convex relaxation is equivalent to solving, to global optimality, a related motion planning problem in a potentially higher-dimensional space.
This geometric interpretation yields necessary and sufficient conditions for tightness, and a clear intuition for when the relaxation is loose.
Second, we show that the relaxation admits a symmetry reduction that makes it significantly smaller than one might expect, with PSD cone sizes that scale linearly with the polynomial degree and are independent of the ambient dimension.
The resulting relaxation is 10 to 100 times faster than direct nonlinear programming transcriptions solved with SNOPT and IPOPT, exhibits significantly lower variance in solve times, and reliably finds a locally optimal path for the original problem.
We demonstrate its effectiveness as a convex steering function in an RRT planner for minimum-snap quadrotor planning with $C^4$ continuous trajectories.

\end{abstract}

\section{Introduction}
\label{sec:introduction}

\begin{figure}[!ht]
\centering
\includegraphics[width=1.0\columnwidth,trim={120 0 20 0},clip]{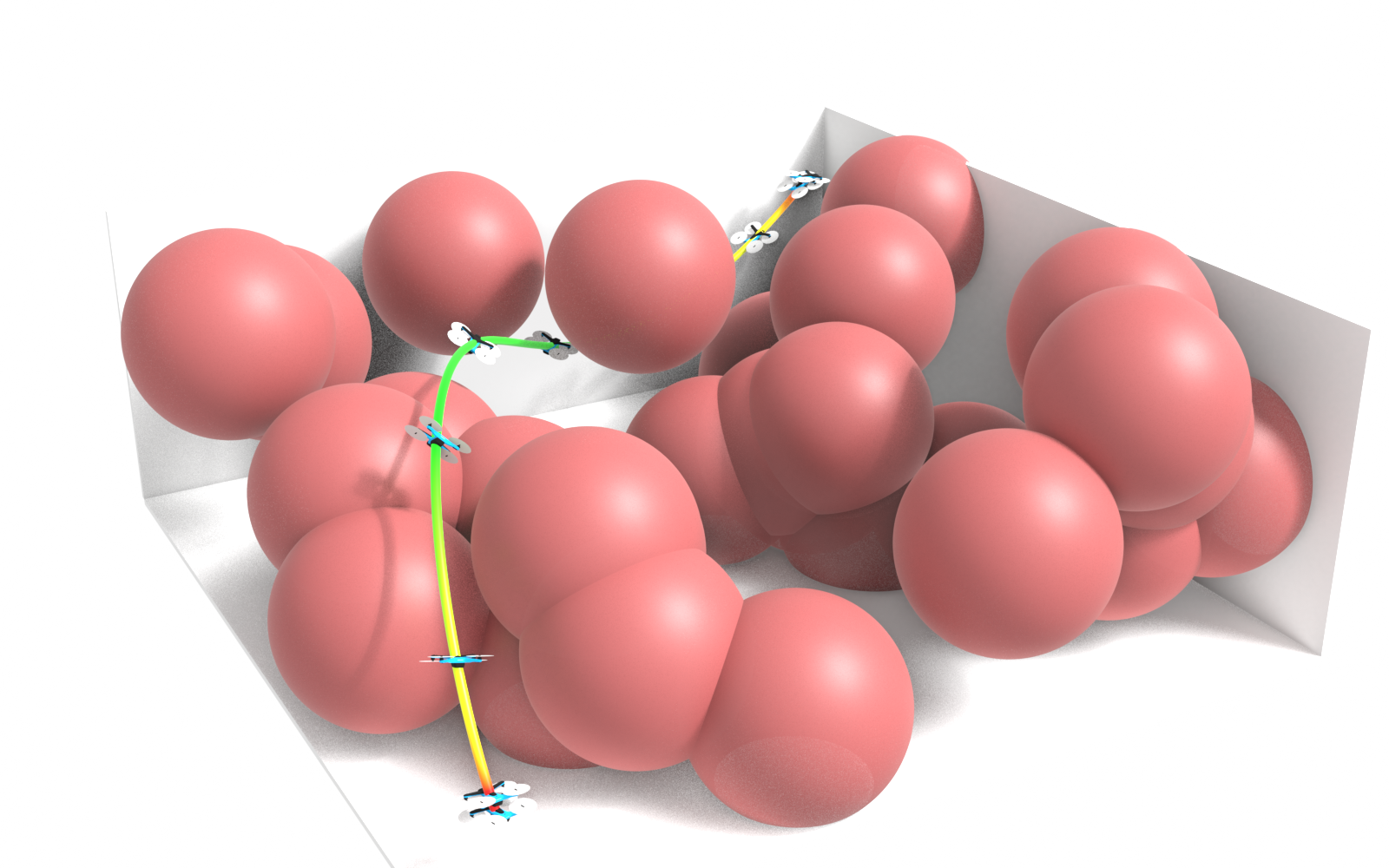}
\caption{
A collision-free, $C^4$-continuous, minimum-snap quadrotor trajectory through a cluttered field of spherical obstacles, planned by an RRT that uses our semidefinite relaxation as a convex steering function.
Our relaxation considers both obstacle-avoidance and other trajectory constraints, while being fast and reliable enough to make optimization a viable inner loop.
The color coding of the trajectory indicates velocity (red is slow, green is fast).
}
\label{fig:hero}
\end{figure}

Collision-free motion planning --
the problem of finding a smooth trajectory between two configurations that avoids a set
of obstacles -- 
is a fundamental challenge in robotics and autonomous systems.
Even in simple settings,
the problem is computationally difficult due to the nonconvexity of the obstacle avoidance
constraints.
In the case of a point robot navigating among sphere obstacles in $\mathbb{R}^2$,
the minimum length path can be computed in polynomial time~\cite{chang2005shortest},
but the problem becomes NP-hard in $\mathbb{R}^3$ for polyhedral
obstacles~\cite{canny1987new},
and the complexity for disjoint unit sphere obstacles in $\mathbb{R}^3$ remains an open
problem~\cite[Section~31.5]{toth2017handbook}.
In robotics, however, minimum path length is rarely the only consideration;
motion planning typically involves costs such as time, energy, or snap, subject to constraints such as boundary conditions, derivative bounds, and continuity constraints.
For this more general problem setting,
no efficient algorithm is known.

Several approaches have been proposed to solve collision-free motion planning in practice.
Optimization-based methods parametrize the trajectory as a piecewise polynomial and naturally handle different costs and constraints.
A popular approach is to use trajectory optimization with \acrfull*{nlp}~\cite{hargraves1987direct,zucker2013chomp,tracy2023differentiable}, but nonconvex optimization can be brittle and sensitive to the initial guess.
Another approach sidesteps the nonconvexity of obstacle avoidance by precomputing a convex decomposition of the free-space~\cite{dai2023certifiedpolyhedraldecompositionscollisionfree,werner2024faster,deits2015computing},
which can itself be a hard problem,
and then plan through the decomposition with convex optimization~\cite{deits2015efficient,marcucci2023motion,marcucci2025biconvex, werner2025superfast, kondo2026mighty}.
This has proven particularly effective for multiquery planning in a single environment.
On the other hand, sampling-based planners are capable of global reasoning~\cite{lavalle1998rapidly,kavraki2002probabilistic,karaman2011sampling},
but are unable to naturally reason about trajectory constraints.
An alternative is to combine sampling-based planners with local trajectory optimization,
either by post-processing a sampled geometric path~\cite{richter2016polynomial},
by precomputing motion primitives offline~\cite{ortiz2024idb},
or by embedding a local solver as a steering function within the sampler~\cite{stoneman2014embedding}.
Geometric sampling with post-processing does not reason about dynamic feasibility during exploration,
often making the sampled path a poor initialization for optimization~\cite{ortiz2024idb}.
Using optimization as the steering function addresses this, but jointly reasoning about obstacles and trajectory constraints is typically nonconvex, and the resulting \acrshort*{nlp} solves can be brittle with highly variable solve times, often making this impractical.

In recent years, semidefinite relaxations have emerged as a powerful tool for solving certain classes of nonconvex optimization problems.
Through the \acrfull*{sos}/moment hierarchy~\cite{parrilo2003semidefinite,lasserre2001global}, one can construct a sequence of increasingly strong relaxations, at the cost of growing computational complexity.
These relaxations use convex optimization, and therefore do not suffer from being brittle or relying on an initial guess.
The application of semidefinite relaxations has seen success in estimation and perception problems such as point cloud registration~\cite{yang2020teaser}, pose-graph SLAM~\cite{rosen2019se} and range-aided SLAM~\cite{papalia2024certifiably}, and has lately been applied to some motion planning problems, such as planning through contact~\cite{graesdal2024towards,kang2025global}.
However, the behavior of semidefinite relaxations for collision-free motion planning is poorly understood, and to our knowledge no theoretical analysis of their properties in this setting has been given.

In this paper,
we study semidefinite relaxations for collision-free motion planning.
In particular, we focus on the problem of a point robot navigating from a start to a goal in a field of spherical obstacles in $\R^n$, minimizing a squared derivative cost (such as minimum energy or minimum snap) subject to path continuity constraints and boundary conditions.
This setting is conceptually simple yet captures the hardness of motion planning.
We formulate this problem for polynomial curves exactly as a nonconvex problem,
and present a natural semidefinite relaxation of it.
Our core contribution is a geometric interpretation of the relaxation: we show it is equivalent to solving, to global optimality, a related collision-free motion planning problem in a space of potentially higher dimension than the original ambient space $\R^n$. This yields necessary and sufficient conditions for when the relaxation is exact, and a clear intuition for when the relaxation is loose.

Beyond the theoretical analysis,
we show that the relaxation provides a fast and reliable convex subroutine for motion planning.
This is enabled by two properties of our formulation.
First,
we use a classical result from \acrshort*{sos} to certify exactly when a polynomial curve avoids a spherical obstacle,
without the need for discretization.
The resulting \acrfull*{lmi} condition takes a particularly nice form,
with \acrshort*{psd} constraints that scale linearly with the polynomial trajectory degree $d$ and are independent of the ambient dimension $n$.
Second,
the \acrshort*{lmi} condition is nonconvex in the trajectory coefficients,
so we formulate a first-order semidefinite relaxation to obtain a convex program.
A naive application of the well-known Shor relaxation~\cite{shor1987quadratic} yields a large program,
but we exploit symmetry to reduce the size of the PSD cone by a factor of $n$.
Together,
these properties yield a relaxation that is 10 to 100 times faster than direct nonlinear programming transcriptions solved with SNOPT~\cite{gillSNOPT2005} and IPOPT~\cite{waechterIPOPT2006},
exhibits significantly lower variance in solve times,
and reliably finds a locally optimal path.
We demonstrate the effectiveness of using our method as a convex and robust steering function in an RRT,
planning $C^4$ continuous,
minimum-snap collision-free trajectories through highly cluttered environments.

\section{Related Works}
\label{sec:related_work}
In this section, we review relevant works on semidefinite relaxations for motion planning and optimal control.
As discussed in the introduction, there is also a substantial body of work on semidefinite relaxations for perception and estimation problems, which we do not review here.

Several works apply semidefinite relaxations to motion planning and control problems and empirically observe tight relaxations, though without analyzing when or why they are tight.
Teng et al.~\cite{teng2023convex} formulate kinodynamic motion planning for rigid body systems as exact polynomial optimization problems, using a variational integrator to discretize the dynamics on Lie groups.
They apply the moment hierarchy and empirically find that tight solutions are obtained at the second order of the hierarchy for most systems they consider.
Similarly, H\"aring et al.~\cite{haring2024trajectory} study minimum-energy control of the unicycle model and empirically show that the second-order relaxation is tight.
In~\cite{vega2025convex}, Vega et al. formulate spacecraft maneuver planning with a single spherical keep-out zone (analogous to considering a single spherical obstacle) as a \acrfull*{qcqp} with linearized dynamics, and apply the Shor relaxation to obtain empirically tight solutions.
Recently, a setting similar to ours was considered by Mahajan et al.~\cite{mahajan2026tinysdp}, who apply the Shor relaxation to collision-free \acrshort*{mpc} with spherical obstacles and discretized collision-avoidance constraints, focusing on 
developing a custom cached Riccati-based \acrshort*{admm} solver that enables real-time rates on embedded hardware.

A related line of work applies semidefinite relaxations to time-scaled optimal control.
Yang et al.~\cite{yang2025new} develop a tailored semidefinite relaxation for linear and piecewise-affine systems and combine it with convex decompositions of the free-space and \acrfull*{gcs}~\cite{marcucci2024shortest} to jointly optimize mode sequences and time-optimal collision-free trajectories.
Dong et al.~\cite{dong2026fast} apply a similar relaxation to jointly optimize trajectories and time allocations for optimal control problems with spatio-temporal constraints, such as a quadrotor required to pass through a sequence of known waypoints within specified time windows, and exploit the banded sparsity of the multiple shooting formulation.

Semidefinite relaxations have also been applied to contact-rich planning.
Graesdal et al.~\cite{graesdal2024towards} relax the contact dynamics of planar pushing with the first-order semidefinite relaxation, strengthen the formulation with implied constraints, and use \acrshort*{gcs} to plan optimal contact mode sequences.
Kang et al.~\cite{kang2025global} apply the moment hierarchy to a broader class of contact-rich planning problems,
encoding the discrete mode sequence directly in the hierarchy rather than through \acrshort*{gcs},
and leverage problem-specific sparsity to obtain solutions at the second or third order of the hierarchy.
Finally, Wei and D\"umbgen~\cite{wei2026global} take a different approach: they sample trajectory rollouts and use KernelSOS, a global optimization technique that builds a nonparametric surrogate of the cost landscape and minimizes it using semidefinite optimization, to find promising regions of the search space that are then refined with a sampling-based local optimizer.

In contrast to these works, which apply semidefinite relaxations to various planning problems and observe tightness empirically, our goal is to understand \emph{when} the relaxation is tight.
We therefore study a deliberately simplified setting, rich enough to capture the core difficulty of collision-free planning yet simple enough to admit a precise analysis.
This setting also reveals a symmetry that reduces the \acrshort*{sdp} to a size independent of the ambient dimension, which can be combined with the correlative sparsity exploited in prior work.

\section{Motion Planning}
\label{sec:motion_planning}
In this section, we state the optimal collision-free motion planning problem exactly, as an optimization over curves where the entire curve must be collision-free.
This problem is infinite-dimensional since the decision variable is a curve, and cannot be solved numerically as stated.
To address this, we restrict the curves to polynomials and define the class of curve costs we consider.
The result is a semi-infinite \acrshort*{qcqp}: the decision variables are finite-dimensional, but the obstacle-avoidance constraints must hold at every point along the curve.
In the next section, we use \acrshort*{sos} to reformulate the collision-avoidance constraint exactly as \acrshort*{lmi} conditions.

\subsection{Problem Statement}

We seek a collision-free path that minimizes the sum of the squared $\ell_2$-norm of its derivatives, such as velocity or snap, subject to boundary conditions and continuity constraints.
Let the trajectory be represented as a curve $q: [0, 1] \rightarrow \R^n$.
We require the trajectory $q$ to be $\eta$-times continuously differentiable, that is, $q \in \mathcal{C}^\eta$, and denote $q^{(i)} = (d^{i}/d s^{i})q$ as the $i$-th derivative.
We consider $m$ spherical obstacles with centers $c_j \in \R^n$ and radii $r_j > 0$, for $j = 1, \ldots, m$.
We formulate the motion planning problem as the following nonconvex optimization problem:
\begin{subequations}\label{eq:path_planning}
\begin{align}
    \min \quad & \int_0^1 \| q^{(k)}(s) \|_2^2 \, ds\\
    \subto \quad
    & \|q(s) - c_j\|_2^2 - r_j^2 
     \ge 0, \quad \forall s \in [0, 1],
    \label{eq:obstacle_const}
    \\
    &\quad\quad\quad\quad \quad \quad \quad \quad \quad\;\ j = 1, \ldots, m, \\
      &q^{(i)}(0) = q_0^{(i)}, 
      \label{eq:motion_planning_bc1}
      \\
      & q^{(i)}(1) = q_1^{(i)}, \quad i = 0, \ldots, \ell
      \label{eq:motion_planning_bc2}
      \\
    & q \in \mathcal{C}^\eta,
\end{align}
\end{subequations}
where $k \in \mathbb{N}_+$ is the derivative order for the cost, and there are boundary conditions up to the $\ell$-th derivative on the curve.
For motion planning, common choices for $k$ include $k=1$ (minimum energy), $k=2$ (minimum acceleration), and $k=4$ (minimum snap).
The problem is nonconvex due to~\eqref{eq:obstacle_const}.
The problem is infinite-dimensional: the decision variable $q$ is a function in $C^\eta$, and for each obstacle $j$, the avoidance condition~\eqref{eq:obstacle_const} represents a continuum of nonconvex constraints indexed by $s \in [0,1]$.
As such, it is unclear how to solve problem~\eqref{eq:path_planning} numerically.

\subsection{Polynomial Curves}
We parametrize the curve as a piecewise polynomial,
where each segment is a polynomial curve $\gamma: [0,1] \rightarrow \R^n$ of degree $d$:
\begin{align}
    \gamma(s) = \Gamma b_d(s),
    \label{eq:poly_curve}
\end{align}
where $b_d(s) : \R \rightarrow  \R^{d+1}$ is a basis of polynomials of degree less than or equal to $d$ in the variable $s$,
and $\Gamma \in \R^{n \times (d+1)}$ is the corresponding coefficient matrix.
With this parametrization,
the decision variables for each segment reduce to the coefficient matrix $\Gamma$,
and~\eqref{eq:path_planning} becomes an optimization problem with a finite number of variables,
albeit still with an infinite number of constraints.

To keep the presentation clean, we develop all formulations for a single segment. This is not a restriction of the formulation, as the cost and collision-avoidance constraints in~\eqref{eq:path_planning} decompose across the polynomial segments, which are coupled only through linear continuity constraints, so the extension to multiple segments is straightforward.
We likewise present these formulations in a basis-independent form, specializing to the Bernstein basis only for our numerical experiments and a few key results.

\subsection{Curve Costs}
\label{sec:curve_costs}
For a polynomial curve $\gamma$ of degree $d$ as defined in~\eqref{eq:poly_curve}, the squared $\mathcal{L}_2$ norm of the $k$-th derivative can be expressed as
\begin{equation}
    \int_0^1 \|\gamma^{(k)}(s)\|_2^2 \, ds = \mathrm{tr}(G^{(k)}_d \, \Gamma^T \Gamma),
    \label{eq:curve_cost_general}
\end{equation}
where 
\begin{align}
G^{(k)}_d = D_k G_{d-k} D_k^T \in \mathbb{S}^{d+1}
\end{align}
is the $k$-th derivative Gram matrix of degree $d$, $D_k \in \R^{(d+1)\times(d+1-k)}$ is the differentiation matrix defined by $b^{(k)}_d(s) = D_kb_{d-k}(s)$, and
\begin{equation}
    G_d = \int_0^1 b_d(s) \, b_d(s)^T \, ds
    \label{eq:gram_matrix_def}
\end{equation}
is the basis-dependent Gram matrix of degree $d$.
Both $G_d$ and $G^{(k)}_d$ can be computed analytically for any choice of basis.
To keep notation light, we write $G$ and $G^{(k)}$ when the degree $d$ is clear from context.
See~\cref{sec:gram_matrix} for explicit expressions for polynomials represented as Bézier curves, i.e. with the Bernstein basis.

The nullspace of $G^{(k)}_d$ characterizes the zero-cost directions of the polynomial trajectory, which we use in our theoretical analysis in \Cref{sec:theoretical_results}.
We present the following standard result on $G^{(k)}_d$ without proof:
\begin{lemma}
    \label{lemma:G_null}
    $G^{(k)}_d$ has rank $d+1-k$, and its nullspace is the space of polynomials of degree at most $k-1$.
    That is, $G^{(k)}_d v = 0$ if and only if $v^T b_d(s)$ is a polynomial of degree at most $k-1$.
\end{lemma}

\subsection{Semi-Infinite QCQP Reformulation}
We reformulate the motion planning problem~\eqref{eq:path_planning} in the polynomial coefficient matrix $\Gamma \in \R^{n \times (d+1)}$.
The result is a semi-infinite \acrshort*{qcqp}, which has a finite number of decision variables, but obstacle avoidance must still hold point-wise for all $s \in [0,1]$.

To keep notation light, we present the derivation for a single unit sphere centered at $c \in \R^n$.
The condition that $\gamma$ avoids the sphere is equivalent to the polynomial nonnegativity condition:
\begin{align}
    p(s) := \|\gamma(s) - c\|_2^2 - 1 \geq 0 \quad \forall s \in [0,1].
    \label{eq:nonneg_collision}
\end{align}
To write the nonnegativity condition~\eqref{eq:nonneg_collision} in terms of the coefficient matrix $\Gamma$,
we first express the shifted curve $\gamma(s) - c$ in the polynomial basis $b_d(s)$.
We introduce a vector $u \in \R^{d+1}$ satisfying $u^T b_d(s) = 1$,
so that $\gamma(s) - c = \Gamma b_d(s) - c \cdot 1 = (\Gamma - cu^T) b_d(s)$.
Using this, together with $1 \cdot 1 = b_d(s)^T uu^T b_d(s)$,~\eqref{eq:nonneg_collision} can be rewritten as the quadratic form
\begin{align}
    p(s)
    &= b_d(s)^T \big((\Gamma - cu^T)^T(\Gamma - cu^T) - uu^T\big) b_d(s)
    \label{eq:nonneg_collision_quadratic}
    \\
    &\geq 0 \quad \forall s \in [0,1],
    \nonumber
\end{align}
which is a univariate polynomial in $s$ of degree $2d$.

Combining the cost, the obstacle condition~\eqref{eq:nonneg_collision_quadratic}, and the boundary conditions, the polynomial motion planning problem is
\begin{align}
\min \quad &\mathrm{tr}(G^{(k)} \Gamma^T \Gamma)
\label{eq:polynomial_nonconvex_prog}
\\
\subto \quad & b_d(s)^T \big((\Gamma - cu^T)^T(\Gamma - cu^T) - uu^T\big) b_d(s) \geq 0
\nonumber
\\
&\quad\quad\quad \forall s \in [0,1], \nonumber \\
& \Gamma A = B,
\nonumber
\end{align}
where $A$ and $B$ encode the boundary conditions~\eqref{eq:motion_planning_bc1} and~\eqref{eq:motion_planning_bc2} on the first $\ell$ derivatives:
\begin{align}
    A &= \begin{bmatrix}
        b_d(0) & b_d(1) & \ldots & b_d^{(\ell)}(0) & b_d^{(\ell)}(1)
    \end{bmatrix}, \nonumber \\
    B &= \begin{bmatrix}
        q_0 & q_1 & \ldots & q_0^{(\ell)} & q_1^{(\ell)}
    \end{bmatrix}.
    \label{eq:bc_matrices}
\end{align}
The problem is naturally written in matrix form.
The cost and obstacle constraint are both quadratic in $\Gamma$, making the problem a semi-infinite \acrshort*{qcqp}.

\section{Certifying Collision-Avoidance}
\label{sec:obstacle_avoidance}
Condition~\eqref{eq:nonneg_collision_quadratic} is the nonnegativity of a univariate polynomial over an interval, which we know how to rewrite exactly using \acrshort*{sos}.
In this section, we show that the resulting conditions take a particularly nice form that is independent of the ambient dimension $n$ of the curve and scales linearly with the polynomial degree $d$.
The conditions are quadratic in the polynomial coefficients $\Gamma$, so for a fixed trajectory, certifying collision-avoidance reduces to an \acrshort*{sdp}.
When searching over trajectories, the conditions are nonconvex, which we address in the next section.
We first derive the condition in a basis-free fashion, before presenting a concrete instance of it for the Bernstein basis.

\subsection{Basis-Independent Derivation}
The key tool is the following classical result from real algebraic geometry:
\begin{theorem*}[Markov-Lukács {\cite[\S 1.21]{szeg1939orthogonal}}]
A univariate polynomial $p \in \R[s]$ of degree $2d$ is nonnegative over an interval $[a,b]$ if and only if
\begin{align}
    p(s) = \sigma_0(s) + (s-a)(b-s) \sigma_1(s),
\end{align}
where $\sigma_0, \sigma_1$ are \acrshort*{sos} with
$\text{deg}(\sigma_0) \leq 2d$ and
$\text{deg}(\sigma_1) \leq 2d - 2$.
\end{theorem*}

Applying this theorem to~\eqref{eq:nonneg_collision_quadratic} immediately gives:

\begin{lemma}[Collision avoidance for polynomial curves]
\label{lem:collision_avoidance}
A polynomial curve $\gamma(s) = \Gamma b_d(s)$ of degree $d$ in $s$ lies entirely outside a unit sphere centered at $c$ if and only if there exist Gram matrices $Q_0 \in \mathbb{S}^{d+1}_+$ and $Q_1 \in \mathbb{S}^{d}_+$ such that
\begin{align}
    \mathcal{L}((\Gamma - cu^T)^T(\Gamma - cu^T), Q_0, Q_1) = 0,
    \label{eq:collision_free_conditions}
\end{align}
where $\mathcal{L}$ is the linear map obtained by matching polynomial coefficients
in the Markov-Lukács decomposition,
and $u \in \R^{d+1}$ satisfies $u^T b_d(s) = 1$.
\end{lemma}

\begin{proof}
By the Markov-Lukács theorem, the nonnegativity condition~\eqref{eq:nonneg_collision_quadratic} is equivalent to
\begin{align}
    p(s) - \sigma_0(s) - s(1-s)\sigma_1(s) = 0,
    \nonumber
\end{align}
where $\sigma_0, \sigma_1$ are \acrshort*{sos} polynomials of degree at most $2d$ and $2d-2$, respectively.
Writing the \acrshort*{sos} multipliers as quadratic forms $\sigma_0(s) = b_d(s)^T Q_0 \, b_d(s)$ and $\sigma_1(s) = b_{d-1}(s)^T Q_1 \, b_{d-1}(s)$ with $Q_0 \succeq 0$ and $Q_1 \succeq 0$, and matching polynomial coefficients defines the linear map $\mathcal{L}$.
\end{proof}

The condition is linear in $(\Gamma - cu^T)^T(\Gamma - cu^T) \in \mathbb{S}^{d+1}_+$
and therefore quadratic in $\Gamma$ itself.
The extension to ellipsoidal obstacles is straightforward:
\begin{corollary}[Extension to ellipsoidal obstacles]
\label{cor:collision_avoidance_ellipsoid}
For a general ellipsoidal obstacle $\set{x \in \R^n \mid (x-c)^T E (x-c) \leq 1}$
where $E$ is a positive semidefinite matrix, \Cref{lem:collision_avoidance} extends to
\begin{align}
    \mathcal{L}((\Gamma - cu^T)^TE(\Gamma - cu^T), Q_0, Q_1) = 0,
    \label{eq:collision_free_conditions_ellipsoid}
\end{align}
where $\mathcal{L}$, $Q_0$, and $Q_1$ are as in \Cref{lem:collision_avoidance}.
\end{corollary}

\subsection{Collision-Avoidance with Bézier Curves}
Bézier curves, which are polynomial curves expressed in the Bernstein basis~\cite{farouki1988algorithms}, are a common choice in motion planning~\cite{marcucci2023motion}.
We instantiate the collision-avoidance conditions of \Cref{lem:collision_avoidance} in the closely related scaled Bernstein basis, which drops the binomial coefficients and yields simpler expressions.
Our implementation uses the standard Bernstein basis for its better numerical conditioning, with the corresponding expressions given in \Cref{sec:bernstein_ids}.

The $i$-th scaled Bernstein basis polynomial of degree $d$ is defined as
\begin{align}
    B_i^d(s) := s^i(1-s)^{d-i}, \quad i = 0, \ldots, d,
\end{align}
for $s \in [0,1]$.
Recall from \Cref{lem:collision_avoidance} that we require a vector $u$ satisfying $u^T b_d(s) = 1$.
By the binomial theorem,
\begin{align}
    \sum_{i=0}^d \binom{d}{i} s^i(1-s)^{d-i} = 1, \nonumber
\end{align}
so we take $u_i = \binom{d}{i}$ for $i = 0, \ldots, d$.

We now make the collision-avoidance condition from \Cref{lem:collision_avoidance} explicit for the scaled Bernstein basis.
The derivation follows standard \acrshort*{sos} coefficient matching,
which takes a particularly simple form in this basis due to the key identity
\begin{align}
    B_i^d(s) \cdot B_j^d(s) = B_{i+j}^{2d}(s),
    \nonumber
\end{align}
which follows immediately from the definition of the scaled Bernstein basis polynomials.
From this, a quadratic form in $b_d(s)$ (like the one in~\eqref{eq:nonneg_collision_quadratic}) can be expressed in the basis $b_{2d}(s)$ as
\begin{align}
    b_d(s)^T M b_d(s)
    &= \sum_{i,j} M_{ij} B_i^d(s) B_j^d(s)
    = \sum_{i,j} M_{ij} B_{i+j}^{2d}(s) \nonumber \\
    &= \mathcal{S}_d(M)^T b_{2d}(s),
    \nonumber
\end{align}
where $\mathcal{S}_d: \mathbb{S}^{d+1} \rightarrow \R^{2d+1}$ sums along the anti-diagonals of $M$:
\begin{align}
    [\mathcal{S}_d(M)]_k = \sum_{\substack{i+j=k}} M_{ij}.
    \label{eq:anti_diag_map}
\end{align}
Similarly, since
\begin{align}
    s(1-s) \cdot B_i^{d-1}(s) \cdot B_j^{d-1}(s) = B_{i+j+1}^{2d}(s), \nonumber
\end{align}
the product $s(1-s) \cdot b_{d-1}(s)^T M b_{d-1}(s)$ can be expressed in the basis $b_{2d}(s)$ as
\begin{align}
    s(1-s) b_{d-1}(s)^T M b_{d-1}(s) = \mathcal{T}_d(M)^T b_{2d}(s), \nonumber
\end{align}
where $\mathcal{T}_d: \mathbb{S}^{d} \rightarrow \R^{2d+1}$ is a shifted version of $\mathcal{S}_{d-1}$:
\begin{align}
    [\mathcal{T}_d(M)]_k = \begin{cases}
        [\mathcal{S}_{d-1}(M)]_{k-1} & \text{if } 1 \leq k \leq 2d-1, \\
        0 & \text{otherwise}.
    \end{cases}
    \label{eq:shifted_map}
\end{align}

Using these maps, we can write the collision-avoidance condition from \Cref{lem:collision_avoidance} explicitly:
\begin{lemma}[Collision avoidance in the scaled Bernstein basis]
\label{lem:collision_avoidance_bernstein}
A polynomial curve $\gamma(s) = \Gamma b_d(s)$ of degree $d$ in the scaled Bernstein basis lies entirely outside a unit sphere centered at $c$ if and only if there exist $Q_0 \in \mathbb{S}^{d+1}_+$ and $Q_1 \in \mathbb{S}^{d}_+$ such that
\begin{align}
    \mathcal{S}_d\big((\Gamma - cu^T)^T(\Gamma - cu^T) - uu^T - Q_0\big) - \mathcal{T}_d(Q_1) = 0,
    \label{eq:collision_free_conditions_bernstein}
\end{align}
where $\mathcal{S}_d$ and $\mathcal{T}_d$ are the linear maps defined in~\eqref{eq:anti_diag_map} and~\eqref{eq:shifted_map}.
\end{lemma}
\begin{proof}
Apply $\mathcal{S}_d$ and $\mathcal{T}_d$ to match coefficients in the basis $b_{2d}(s)$ in the proof of \Cref{lem:collision_avoidance}.
\end{proof}

For Bézier curves (i.e., the standard Bernstein basis),
the same condition holds with $u = e$, the vector of ones, and the corresponding maps $\mathcal{S}_d$,
$\mathcal{T}_d$ given in \Cref{sec:bernstein_ids}.

\subsection{The Special Case of Line Segments}
For line segments, the \acrshort*{lmi} condition reduces to a \acrfull*{rsoc} constraint, which is computationally much cheaper.
The natural parameterization of a line segment from $\gamma_0$ to $\gamma_1$,
\begin{align}
    \gamma(s) = (1-s) \gamma_0 + s \gamma_1, \quad s \in [0,1],
\end{align}
is a Bézier curve of degree $d = 1$, for which the scaled and standard Bernstein bases coincide with $b_1(s) = (1-s, s)$ and $\Gamma = \bmat{\gamma_0 & \gamma_1}$.
Applying \Cref{lem:collision_avoidance_bernstein} with $d = 1$ gives:
\begin{corollary}[Collision avoidance for line segments]
\label{cor:line_segment}
A line segment from $\gamma_0$ to $\gamma_1$ lies entirely outside an ellipsoidal obstacle $\set{x \in \R^n \mid (x-c)^T E (x-c) \leq 1}$ if and only if there exists a scalar $\alpha \geq 0$ such that
\begin{align}
    \bmat{\gamma_0^T \\ \gamma_1^T} E \bmat{\gamma_0 & \gamma_1} \succeq \bmat{1 & 1 + \alpha \\ 1 + \alpha & 1},
    \label{eq:line_segment_lmi}
\end{align}
which is a $2 \times 2$ \acrshort*{psd} constraint that can be encoded as a \acrshort*{rsoc} constraint.
\end{corollary}
\begin{proof}
Direct computation from \Cref{lem:collision_avoidance_bernstein} with $d = 1$, eliminating $Q_0$ and letting $\alpha = Q_1$.
\end{proof}

\section{Semidefinite Relaxation}
\label{sec:first_order_relaxation}
Using the results of the previous section, we now rewrite the semi-infinite \acrshort*{qcqp} as a finite-dimensional nonconvex problem, before deriving a semidefinite relaxation of it.

\subsection{Nonconvex Problem}

Substituting the \acrshort*{lmi} condition of \Cref{lem:collision_avoidance} for the continuum constraint in the semi-infinite \acrshort*{qcqp}~\eqref{eq:polynomial_nonconvex_prog} gives a finite-dimensional nonconvex problem:
\begin{align*}
\tag{P}
\label{eq:nonconvex_path_planning}
    \min \quad &\mathrm{tr}(G^{(k)} \, \Gamma^T \Gamma) \\
    \subto
    \quad & \mathcal{L}((\Gamma - cu^T)^T (\Gamma - cu^T),
    Q_0, Q_1) = 0,
    \\
    \quad & Q_0 \succeq 0, \; Q_1 \succeq 0, \\
    &\Gamma A = B,
\end{align*}
where the decision variables are $\Gamma \in \R^{n \times (d+1)}$, $Q_0 \in \mathbb{S}^{d+1}$ and $Q_1 \in \mathbb{S}^d$ (the \acrshort*{sos} Gram matrices from~\eqref{eq:collision_free_conditions}).
The first constraint in~\eqref{eq:nonconvex_path_planning} is a quadratic equality constraint in the polynomial coefficients, and is therefore nonconvex.

\subsection{Forming the Relaxation}
The nonconvexity in~\eqref{eq:nonconvex_path_planning} comes from the first argument to $\mathcal{L}$, which expands as
\begin{align}
    (\Gamma - cu^T)^T (\Gamma - cu^T) &= \Gamma^T \Gamma - (\Gamma^T c) u^T
    \label{eq:quadratic_expansion}
    \\
    &\quad - u (\Gamma^T c)^T + (c^T c) uu^T. \nonumber
\end{align}
We define the Gram matrix of the polynomial coefficients as $X := \Gamma^T \Gamma \in \mathbb{S}^{d+1}$.
Substituting into~\eqref{eq:quadratic_expansion}, the right-hand side becomes affine in $X$ and $\Gamma$.
Since $\mathcal{L}$ is a linear map, the constraint remains affine in the decision variables $X$, $\Gamma$, $Q_0$, and $Q_1$.
The same substitution makes the cost linear in $X$: $\mathrm{tr}(G^{(k)} \Gamma^T \Gamma) = \mathrm{tr}(G^{(k)} X)$.

The only remaining nonconvexity is the constraint $X = \Gamma^T \Gamma$, which 
we relax into the convex constraint $X \succeq \Gamma^T \Gamma$.
Using the Schur complement, we can rewrite the condition as an \acrshort*{lmi}:
\begin{align}
    X \succeq \Gamma^T \Gamma
    \iff
    \bmat{I_n & \Gamma \\ \Gamma^T & X} \succeq 0,
    \label{eq:schur_lmi}
\end{align}
where $I_n$ is the $n$-dimensional identity matrix.

With this, the convex relaxation of~\eqref{eq:nonconvex_path_planning} is:
\begin{align*}
\tag{SDP}
\label{eq:first_order_relaxation_single}
    \min \quad &\mathrm{tr}(G^{(k)} X) \\
    \subto \quad
    & \mathcal{L}(X - (\Gamma^T c) u^T - u (\Gamma^T c)^T \nonumber \\
    & \quad\quad + (c^T c) uu^T, Q_0, Q_1) = 0, \\
    & Q_0 \succeq 0, \; Q_1 \succeq 0, \\
    & \Gamma A = B \\
    & X A = \Gamma^T B \\
    & \bmat{I_n & \Gamma \\ \Gamma^T & X} \succeq 0.
\end{align*}
This is an \acrshort*{sdp} in the decision variables $\Gamma$, $X$, $Q_0$, and $Q_1$.
The constraints $XA = \Gamma^T B$ are implied by $\Gamma A = B$ when $X = \Gamma^T \Gamma$ holds (left-multiply by $\Gamma^T$), so they are redundant for~\eqref{eq:nonconvex_path_planning} but tighten the relaxation.

\subsection{Strict Feasibility and Facial Reduction}
\label{sec:facial_reduction}
Every feasible point of the \acrshort*{sdp} satisfies

\begin{align}
\left\langle \bmat{B \\ -A} \bmat{B^T & -A^T},\, \bmat{I_n & \Gamma \\ \Gamma^T & X} \right\rangle = 0, \nonumber
\end{align}
so the feasible set lies in a proper face of the \acrshort*{psd} cone.
As a result, the \acrshort*{sdp} is never strictly feasible, and solving it directly leads to numerical difficulties.
In practice, we address this by eliminating the equality constraints before solving, which restricts the \acrshort*{lmi} to this face. This is a facial reduction~\cite{permenter2017reduction} where the reducing certificate is available in closed form.
Eliminating the equality constraints amounts to parameterizing the affine subspace defined by $\Gamma A = B$~\cite[Section 10.1.2]{boydConvexOptimization}, and for Bézier curves in the Bernstein basis, the boundary conditions admit a particularly simple parametrization.
The derivatives at each endpoint depend in a triangular manner on the adjacent control points $\gamma(0) = \gamma_0$, $\gamma'(0)$ depends on $\gamma_0$ and $\gamma_1$, $\gamma''(0)$ on $\gamma_0, \gamma_1, \gamma_2$, and so on (and similarly at $s = 1$, reading from the end).
Prescribing derivatives up to $\ell$ at each endpoint therefore fixes the $\ell{+}1$ adjacent control points as known constants, leaving the remaining $d + 1 - 2(\ell{+}1)$ control points free.
We can therefore write $\Gamma = B D_\text{fixed} + \Gamma_\text{free} D_\text{free}$, where $\Gamma_\text{free}$ collects the free control points and $D_\text{fixed}, D_\text{free}$ are constant matrices that place the boundary data and the free columns into the corresponding columns of $\Gamma$.
Defining the smaller Gram matrix $X_\text{free} := \Gamma_\text{free}^T \Gamma_\text{free}$ and relaxing it into $X_\text{free} \succeq \Gamma_\text{free}^T \Gamma_\text{free}$, the \acrshort*{lmi}~\eqref{eq:schur_lmi} reduces in size from $n + d + 1$ to $n + d + 1 - 2(\ell{+}1)$.
The boundary and tightening constraints are then satisfied by construction, while the cost and collision-avoidance constraints remain linear in $(\Gamma_\text{free}, X_\text{free})$.

\subsection{Extension to Multiple Segments}
For $N$ segments with polynomial coefficient matrices $\Gamma_1, \ldots, \Gamma_N$ and $\mathcal{C}^\eta$ continuity,
the relaxation is derived analogously.
Specifically, it introduces per-segment Gram matrices $X_i$
and cross-Gram matrices $X_{i,i+1}$
between consecutive segments,
relaxing the nonconvex constraint
\begin{align}
\bmat{X_i & X_{i,i+1} \\ X_{i,i+1}^T & X_{i+1}}
=
\bmat{\Gamma_i & \Gamma_{i+1}}^T \bmat{\Gamma_i & \Gamma_{i+1}}
\nonumber
\end{align}
via an \acrshort*{lmi} analogous to~\eqref{eq:schur_lmi}:
\begin{align}
\bmat{I_n & \Gamma_i & \Gamma_{i+1} \\ \Gamma_i^T & X_i & X_{i,i+1} \\ \Gamma_{i+1}^T & X_{i,i+1}^T & X_{i+1}} \succeq 0
\end{align}
Boundary conditions remain unchanged, and $\mathcal{C}^\eta$ continuity between adjacent segments introduces additional linear constraints on $\Gamma_i$ of the form $\Gamma_i A_1 = \Gamma_{i+1} A_0$, where $A_1, A_0 \in \R^{(d+1) \times (\eta+1)}$ collect the endpoint values of the Bernstein basis derivatives at $s = 1$ and $s = 0$ respectively.
Every such linear constraint implies additional tightening constraints on $X_i$ and $X_{i,i+1}$, obtained by left-multiplying by $\Gamma_i^T$ or $\Gamma_{i+1}^T$.
The facial reduction from the previous subsection extends directly to this setting: continuity contributes $\eta{+}1$ additional dependent columns to each pairwise \acrshort*{lmi}, reducing its size from $n + 2(d{+}1)$ to $n + 2(d{+}1) - (\eta{+}1)$ at interior junctions, with a further $\ell{+}1$ removed at pairs containing a global endpoint.

\section{Theoretical Results}
\label{sec:theoretical_results}
We now present our main theoretical results.
We first give a geometric interpretation of the relaxation: solving \eqref{eq:first_order_relaxation_single} is equivalent to solving, to global optimality, a planning problem in a higher-dimensional space.
We then use this interpretation to characterize exactly when the relaxation is tight, and give geometric conditions under which the minimizer of the original nonconvex problem is recovered.
Finally, we show that our relaxation is a symmetry-reduced version of the Shor relaxation of~\eqref{eq:nonconvex_path_planning} under an $O(n)$ symmetry.
As the rest of the paper can be understood without this last result, we defer it to \Cref{sec:symmetry}.

\subsection{Geometric Interpretation}
Recall that \eqref{eq:first_order_relaxation_single} is obtained by an algebraic lift, replacing the nonconvex term $\Gamma^T \Gamma$ in~\eqref{eq:nonconvex_path_planning} with the matrix variable $X \succeq \Gamma^T \Gamma$.
We now give this lift a geometric interpretation.
We introduce a family of higher-dimensional, nonconvex problems~\eqref{eq:nonconvex_path_planning_lifted}, $\rho \in \N$, that allow the path $\rho$ extra spatial dimensions beyond the original $n$ ambient dimensions, while the obstacles and boundary conditions remain in $\R^n$.
We then show that \eqref{eq:first_order_relaxation_single} attains the optimal cost over this family, and that an optimal $(n+\rho)$-dimensional path can be recovered from any minimizer.

Consider extending the problem~\eqref{eq:nonconvex_path_planning} from $\R^n$ to $\R^{n+\rho}$, for some $\rho \in \N$, in the following sense:
Let $v(s) = V b_d(s) \in \R^\rho$ be a polynomial curve in the extra $\rho$ dimensions,
with coefficient matrix $V \in \R^{\rho \times (d+1)}$ 
so the full curve is $(\gamma(s), v(s)) \in \R^{n+\rho}$.
Let the obstacle be a sphere in $\R^{n+\rho}$ centered at $(c, 0)$.
The squared distance to the obstacle center then splits as $\|\gamma(s) - c\|^2 + \|v(s)\|^2$, so the quadratic form~\eqref{eq:nonneg_collision_quadratic} extends with an additional $V^T V$ term.
We require $v$ and its first $\ell$ derivatives to vanish at the boundary: $v^{(i)}(0) = v^{(i)}(1) = 0$ for $i = 0, \ldots, \ell$, or equivalently $V A = 0$.
The resulting problem, with $\Gamma$ as in~\eqref{eq:nonconvex_path_planning}, is:
\begin{align*}
\tag{P$_\rho$}
\label{eq:nonconvex_path_planning_lifted}
    \min \quad &\mathrm{tr}(G^{(k)} (\Gamma^T\Gamma + V^TV)) \\
    \subto
    \quad & \mathcal{L}((\Gamma - cu^T)^T(\Gamma - cu^T) + V^TV, Q_0, Q_1) = 0, \\
    \quad & Q_0 \succeq 0, \; Q_1 \succeq 0, \\
    & \Gamma A = B, \quad V A = 0.
\end{align*}
This is still a nonconvex problem, and~\eqref{eq:nonconvex_path_planning_lifted} reduces to~\eqref{eq:nonconvex_path_planning} for $\rho=0$.
\Cref{fig:tight} and \cref{fig:not_tight} show two different examples with \eqref{eq:nonconvex_path_planning} (left) and (P$_{\rho=1}$) (right).

We now present three results connecting \eqref{eq:first_order_relaxation_single} to the lifted problems \eqref{eq:nonconvex_path_planning_lifted}.
First, the optimal cost of \eqref{eq:nonconvex_path_planning_lifted} is monotonically nonincreasing in $\rho$.
Second, \eqref{eq:first_order_relaxation_single} lower-bounds the optimal cost of every \eqref{eq:nonconvex_path_planning_lifted}.
Third, this lower bound is attained: from any minimizer of \eqref{eq:first_order_relaxation_single}, we can construct a $\rho$ and a pair $(\Gamma, V)$ that minimizes~\eqref{eq:nonconvex_path_planning_lifted}.
Let $c_\text{P}^*$, $c_{\text{P}_\rho}^*$, and $c_\text{SDP}^*$ denote the optimal costs of \eqref{eq:nonconvex_path_planning}, \eqref{eq:nonconvex_path_planning_lifted}, and \eqref{eq:first_order_relaxation_single}.
The \acrshort*{sos} certificates $Q_0, Q_1$ appear identically in all three programs, so we omit them from feasible points to keep notation light.

Our first result formalizes the intuition that~\eqref{eq:nonconvex_path_planning_lifted} gets easier as $\rho$ increases, since every path feasible for \eqref{eq:nonconvex_path_planning_lifted} remains feasible for (P$_{\rho+1}$):
\begin{theorem}
\label{thm:lifted_monotonicity}
The optimal costs of~\eqref{eq:nonconvex_path_planning_lifted} are monotonically nonincreasing: $c_{\text{P}_\rho}^* \geq c_{\text{P}_{\rho+1}}^*$ for all $\rho \in \N$.
\end{theorem}
However, these costs cannot decrease indefinitely, as \eqref{eq:first_order_relaxation_single} lower-bounds the entire family:
\begin{theorem}
\label{thm:lower_bound}
For all $\rho \in \N$, $c_{\text{P}_\rho}^* \geq c_\text{SDP}^*$.
\end{theorem}

Our third result shows that this lower bound is attained at a finite $\rho$, constructed from any minimizer of \eqref{eq:first_order_relaxation_single}:
\begin{theorem}
\label{thm:geometric_interpretation}
Let $(\Gamma^*, X^*)$ be optimal for~\eqref{eq:first_order_relaxation_single} and let $\rho = \operatorname{rank}(X^* - (\Gamma^*)^T \Gamma^*)$.
Then for any $V^*$ satisfying $(V^*)^T V^* = X^* - (\Gamma^*)^T \Gamma^*$,
the pair $(\Gamma^*, V^*)$ is optimal for~\eqref{eq:nonconvex_path_planning_lifted}, and $c_{\textnormal{P}_\rho}^* = c_\textnormal{SDP}^*$.
\end{theorem}

Together, \Cref{thm:lifted_monotonicity,thm:lower_bound,thm:geometric_interpretation} give a clean geometric interpretation of the relaxation: \eqref{eq:first_order_relaxation_single} solves the easiest problem in the family of higher-dimensional problems, in the sense that it attains the lowest optimal cost, $c_\text{SDP}^* = \min_{\rho \in \N} c_{\text{P}_\rho}^*$.
A numerical example for $\R^2$ is shown in \Cref{fig:lifted_interpretation}.

While the number of extra dimensions used by the lifted problem, $\rho = \operatorname{rank}(X - \Gamma^T\Gamma)$, is already bounded by $d+1$ since $X - \Gamma^T\Gamma \in \mathbb{S}^{d+1}$, the boundary conditions give a tighter bound:
\begin{lemma}
\label{lem:rho_bound}
The dimension $\rho$ in \Cref{thm:geometric_interpretation} satisfies $\rho \le (d+1) - 2(\ell+1)$, the number of free control points.
\end{lemma}
In practice, we find $\rho$ to be smaller.
In~\cref{sec:empirical_analysis}, we show that across experiments in $\R^2$, $\R^3$, and $\R^5$, we never observe $\rho > 1$, suggesting that, for the random instances we solve, a single extra dimension already attains the lowest cost in the family.

\begin{figure*}[t]
    \centering
    \begin{subfigure}[t]{0.49\textwidth}
        \centering
        \includegraphics[width=0.45\linewidth]{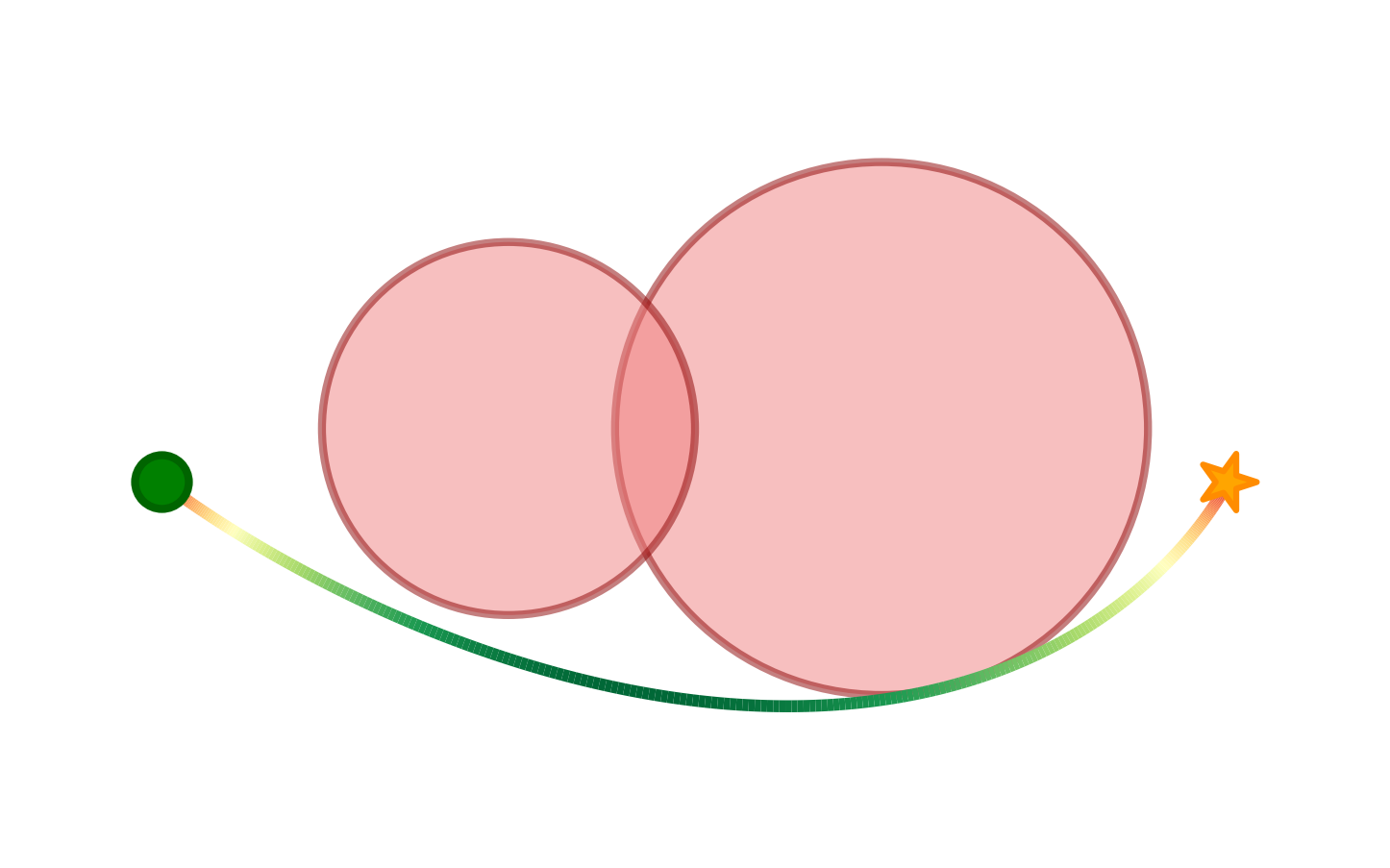}\hfill
        \includegraphics[width=0.54\linewidth]{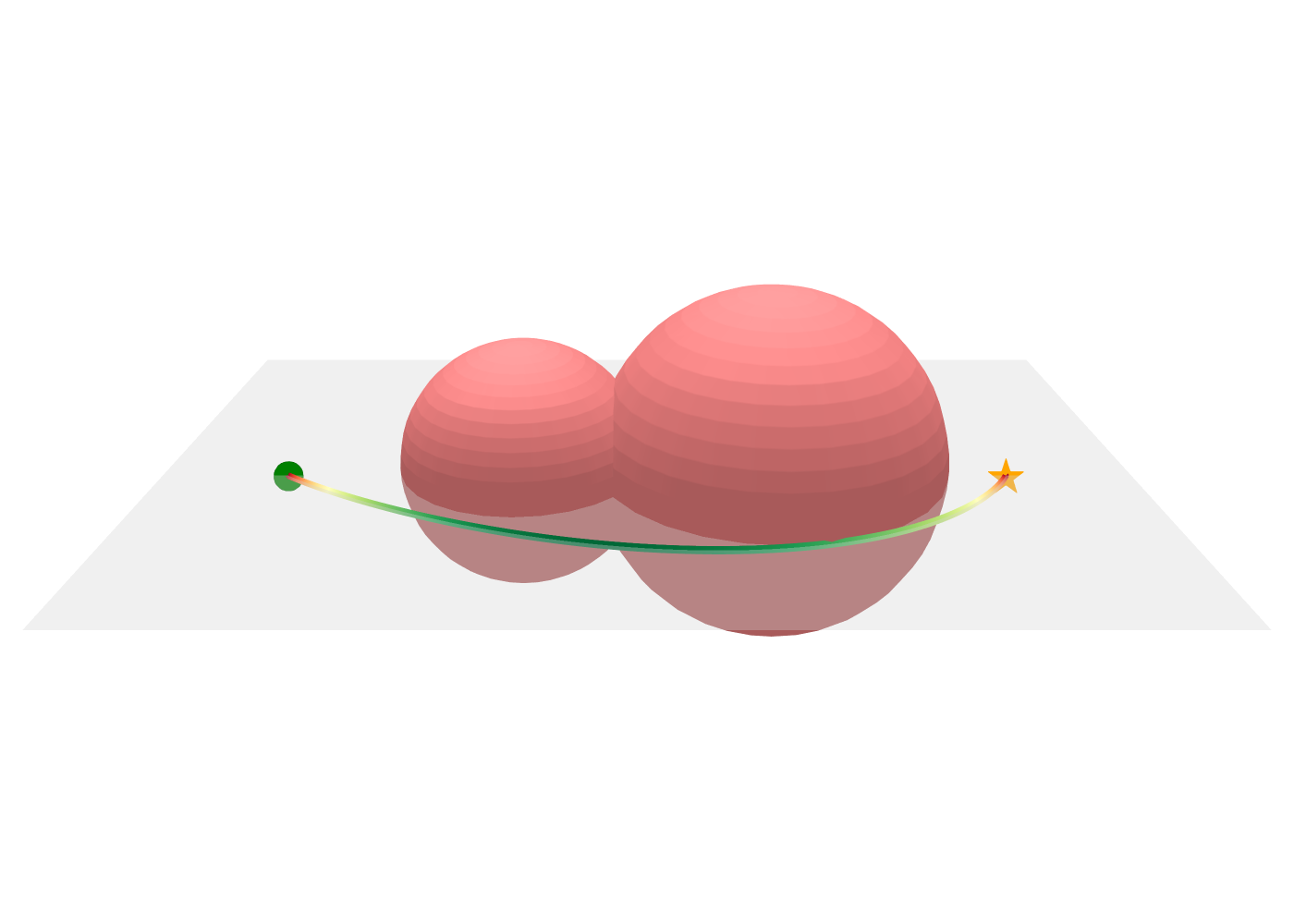}
        \caption{Tight example.}
        \label{fig:tight}
    \end{subfigure}
    \hfill
    \begin{subfigure}[t]{0.49\textwidth}
        \centering
        \includegraphics[width=0.45\linewidth]{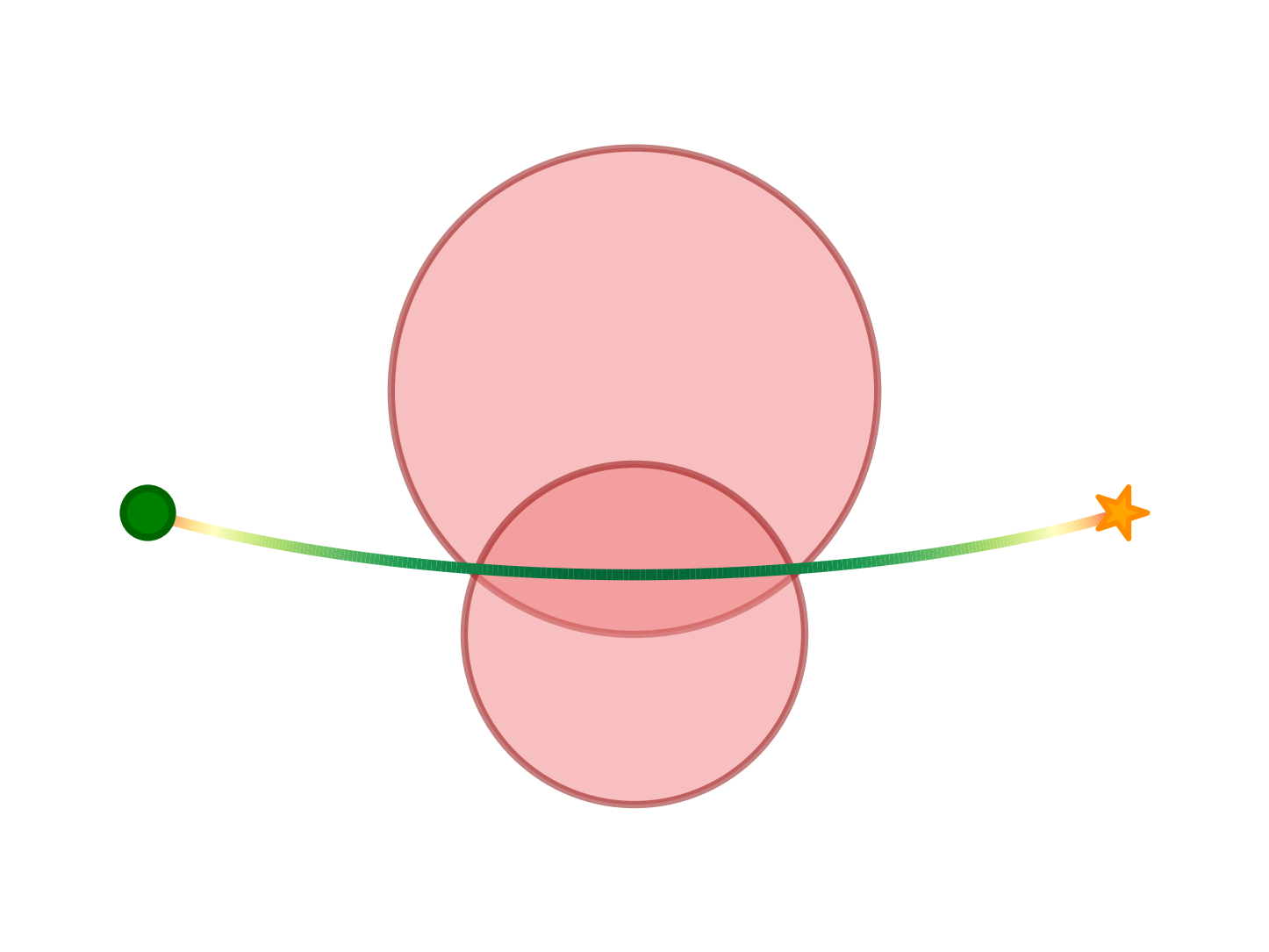}\hfill
        \includegraphics[width=0.54\linewidth]{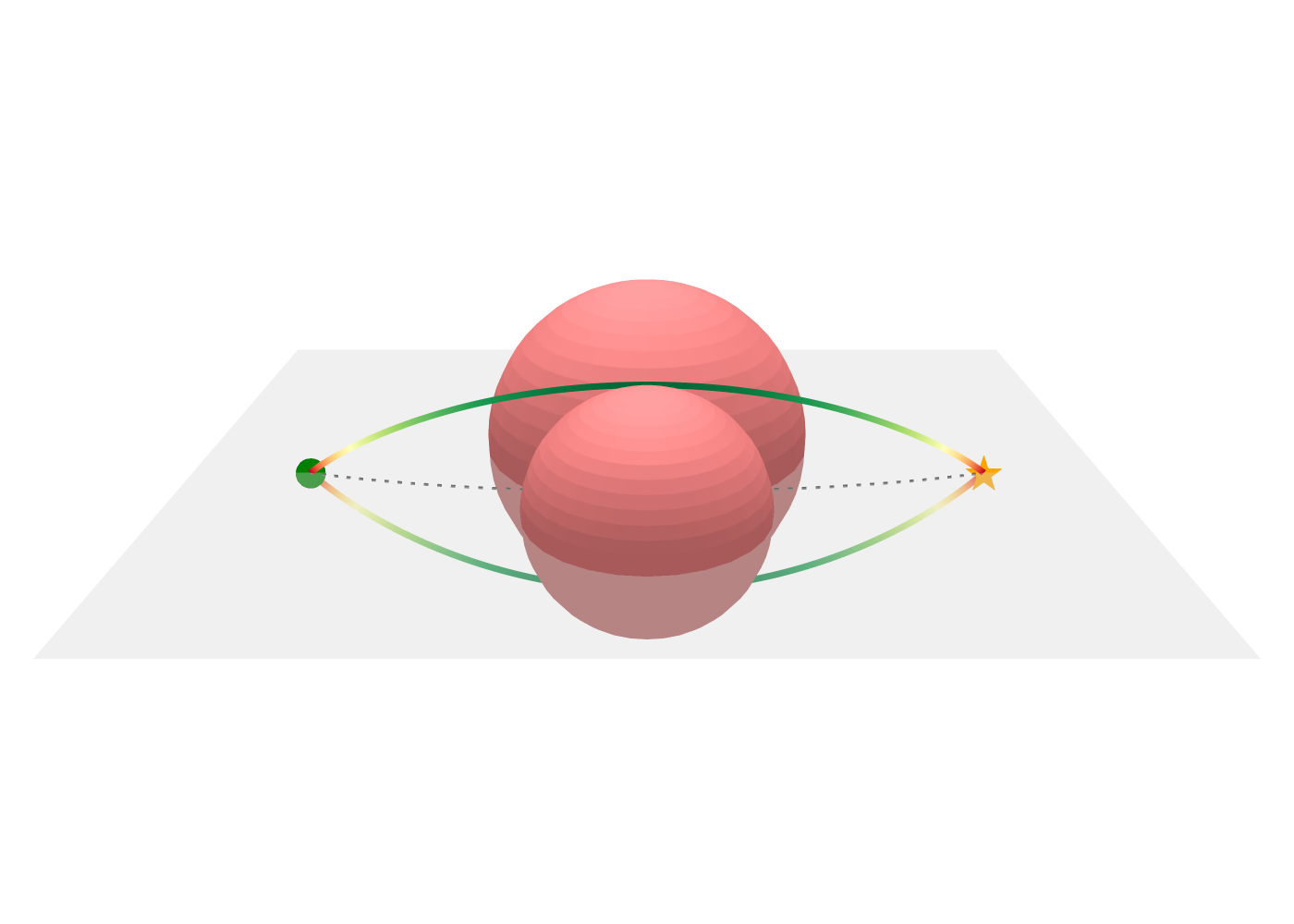}
        \caption{Not tight example.}
        \label{fig:not_tight}
    \end{subfigure}
    \caption{
Geometric interpretation of the relaxation for two examples in $\R^2$.
Each example shows the trajectory defined by the polynomial coefficients $\Gamma$ from the solution to the relaxation in the original ambient space $\R^2$ (left),
and the lifted trajectory defined by $(\Gamma, V)$ recovered from any factorization $V^T V = X - \Gamma^T \Gamma$, following \Cref{thm:geometric_interpretation} (right).
(a) Tight: $X = \Gamma^T \Gamma$, so $V = 0$ and the lifted trajectory coincides with the optimal $\R^2$ trajectory.
(b) Not tight: $\rho = \operatorname{rank}(X - \Gamma^T \Gamma) = 1$, so $V \in \R^{1 \times (d+1)}$ is nonzero and the lifted trajectory uses the extra dimension to find a lower-cost trajectory that, when projected back to $\R^2$, passes through the obstacles. The two trajectories shown are the two factorizations $\pm V$ of $X - \Gamma^T \Gamma$.}
    \label{fig:lifted_interpretation}
\end{figure*}

\subsection{When is the Relaxation Tight?}
We now present a necessary and sufficient condition for tightness, which follows immediately from the geometric interpretation: since \eqref{eq:first_order_relaxation_single} attains the minimum cost over the higher-dimensional problems~\eqref{eq:nonconvex_path_planning_lifted}, $\rho \in \N$, it is tight exactly when this minimum equals the cost of~\eqref{eq:nonconvex_path_planning}.
\begin{theorem}
\label{thm:tightness}
The optimal cost of the relaxation equals the true optimal cost, $c_\text{SDP}^* = c_\text{P}^*$, if and only if $c_\text{P}^* = c_{\text{P}_\rho}^*$ for all $\rho \in \N$.
\end{theorem}
In other words, the relaxation is tight exactly when extending the original problem~\eqref{eq:nonconvex_path_planning} to a higher dimension cannot produce a cheaper collision-free path.

\subsection{Recovering a Minimizer}

\Cref{thm:tightness} characterizes when the relaxation achieves the correct optimal cost, but does not say anything about the minimizer returned by the relaxation.
The next result says that when a unique trajectory is optimal for all the higher-dimensional problems \eqref{eq:nonconvex_path_planning_lifted}, the relaxation recovers it exactly:

\begin{theorem}
\label{thm:uniqueness}
If for all $\rho \in \N$,
every minimizer of~\eqref{eq:nonconvex_path_planning_lifted} has $\Gamma = \Gamma^*$ and $V = 0$,
then every minimizer of~\eqref{eq:first_order_relaxation_single} has $\Gamma = \Gamma^*$ and $X = (\Gamma^*)^T\Gamma^*$.
\end{theorem}
The hypothesis implies $c_{\text{P}_\rho}^* = c_{\text{P}}^*$ for all $\rho$, i.e., the relaxation is tight by \Cref{thm:tightness}.
The uniqueness assumption is stronger but additionally guarantees recovery of the correct minimizer.
It requires uniqueness only in $\Gamma$ and $V$, and $Q_0$ and $Q_1$ need not be unique.

The hypothesis of \Cref{thm:uniqueness} can fail even when the relaxation is tight.
We identify two sufficient conditions for non-uniqueness of~\eqref{eq:nonconvex_path_planning_lifted}.
The first is that the obstacle-free problem has multiple minimizers.
Without obstacle constraints,~\eqref{eq:nonconvex_path_planning_lifted} reduces to a convex QP over $\Gamma$ (since $V = 0$ is then optimal), whose cost has zero-cost directions, which by \Cref{lemma:G_null} are the polynomials of degree below $k$.
The QP has a unique minimizer when the boundary conditions eliminate these directions, which holds when $k \leq 2(\ell+1)$.
This is a standard result in polynomial trajectory optimization, and follows from the uniqueness condition for convex QPs~\cite[Section 10.1.1]{boydConvexOptimization} and the Hermite interpolation theorem~\cite[Section 2.1.5]{stoerNumericalAnalysis}.
The second is the obstacle geometry: even when the obstacle-free problem has a unique minimizer, the nonconvex collision-avoidance constraints can create multiple isolated optima, such as symmetric paths around an obstacle.
In practice, a small random perturbation to the problem geometry typically breaks this symmetry.

\section{Empirical Analysis}
\label{sec:empirical_analysis}
\begin{figure*}[!t]
\centering
\begin{subfigure}[b]{0.32\textwidth}
    \centering
    \includegraphics[height=6cm,width=\linewidth,keepaspectratio]{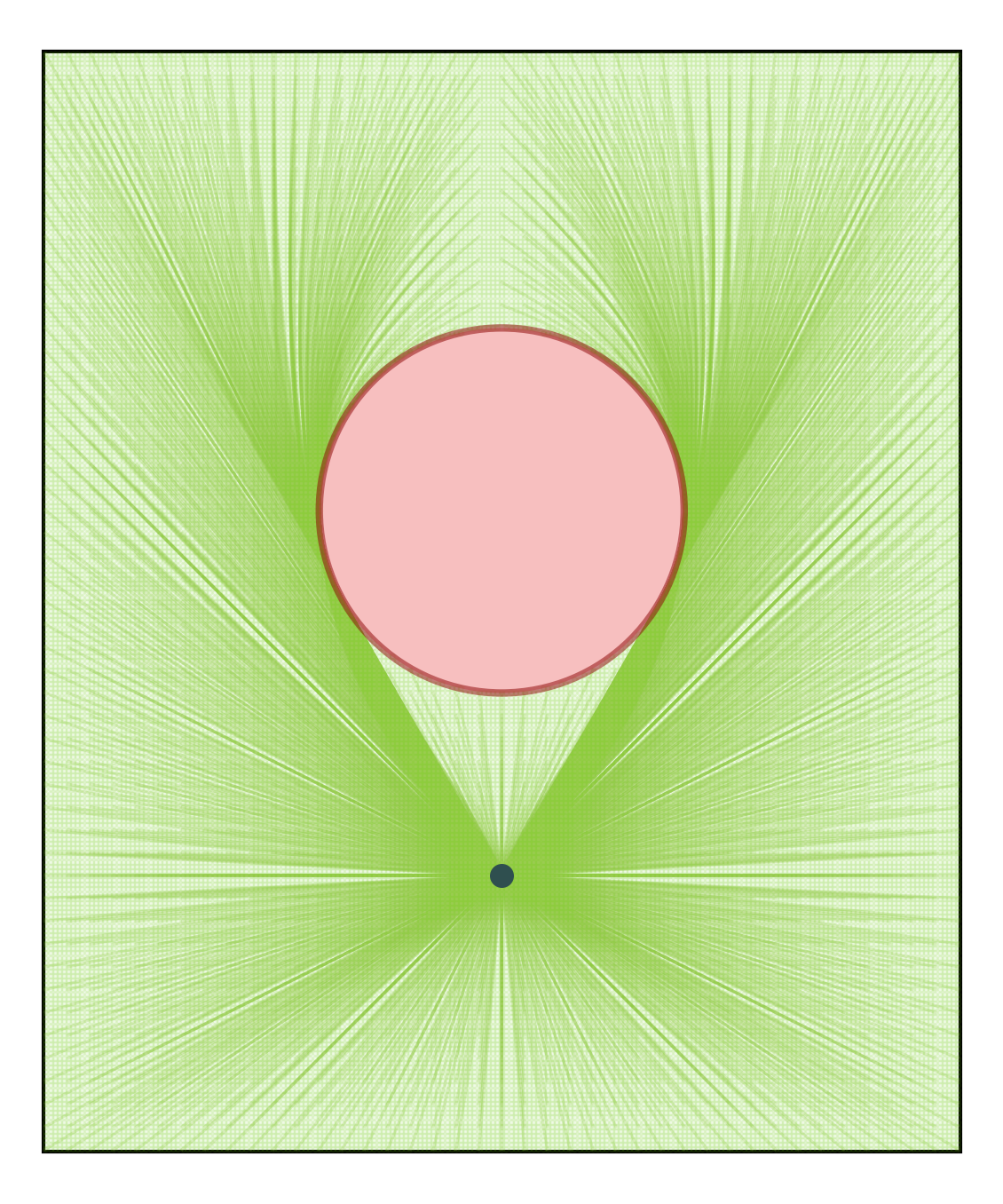}
\end{subfigure}
\hfill
\begin{subfigure}[b]{0.32\textwidth}
    \centering
    \includegraphics[height=6cm,width=\linewidth,keepaspectratio]{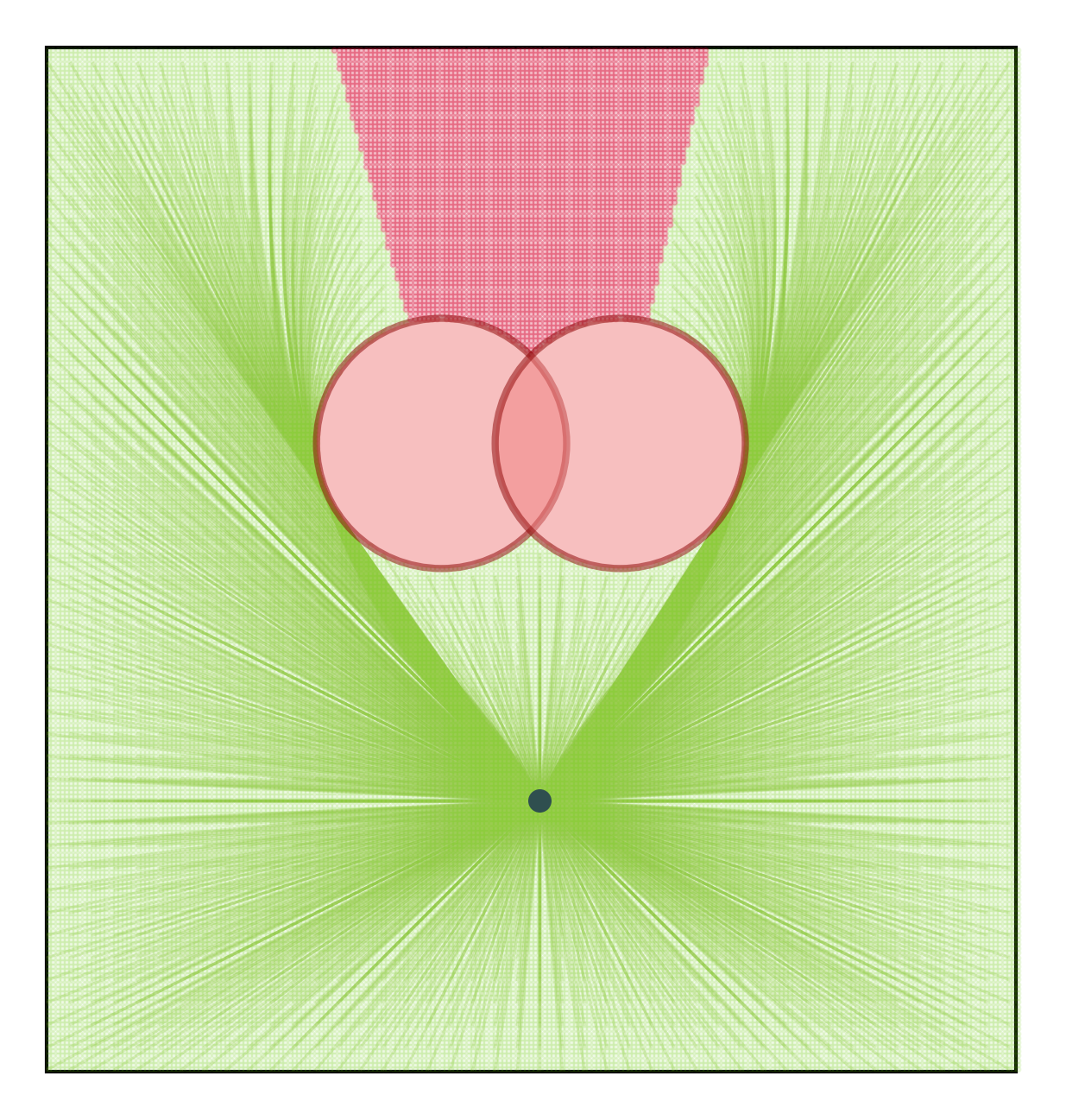}
\end{subfigure}
\hfill
\begin{subfigure}[b]{0.32\textwidth}
    \centering
    \includegraphics[height=6cm,width=\linewidth,keepaspectratio]{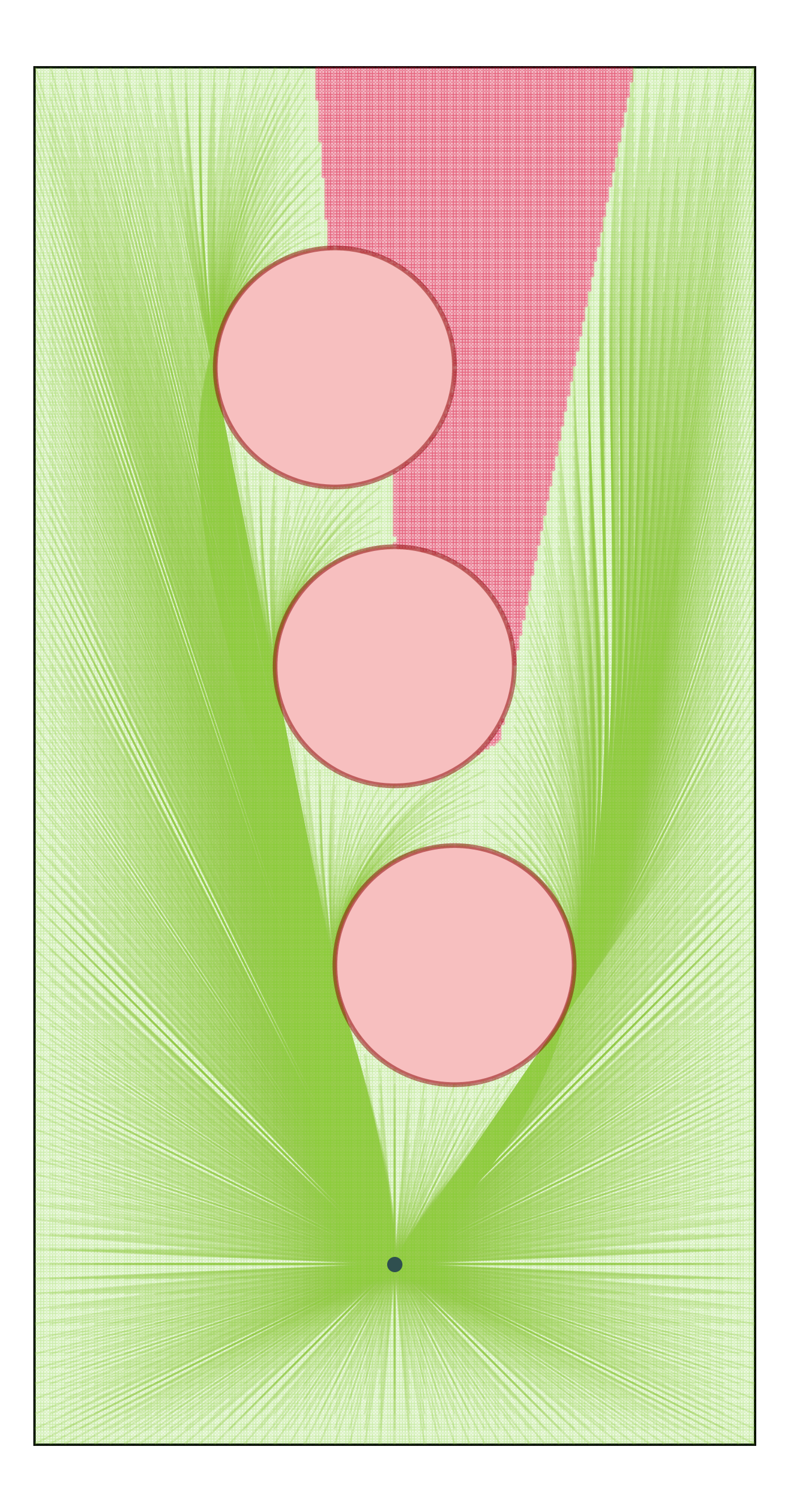}
\end{subfigure}
\caption{
Numerical evaluation of where the relaxation is tight and where it is loose for some representative obstacle configurations.
The green region shows the set of reachable positions from the dark dot for which the relaxation is tight, and the red region shows where it is loose.
The green lines show the trajectories obtained from solving the relaxation.
Results are shown for minimum-energy cost with degree $3$ polynomial trajectories.
}
\label{fig:tightness_geometry_2d}
\end{figure*}

\begin{figure*}[!t]
\centering
\begin{subfigure}[c]{0.48\textwidth}
    \centering
    \includegraphics[width=\linewidth]{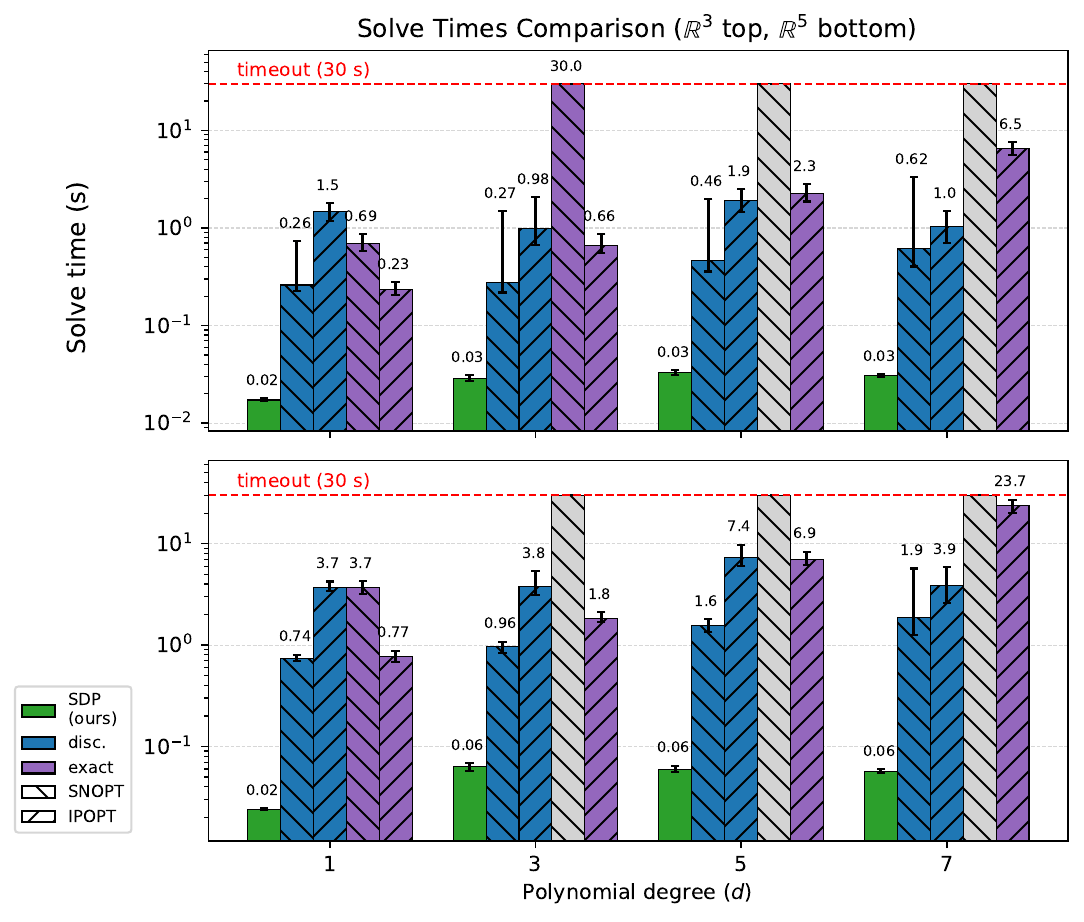}
    \caption{}
    \label{fig:solve_times_plot}
\end{subfigure}
\hfill
\begin{minipage}[c]{0.48\textwidth}
    \centering
    \begin{subfigure}[b]{\linewidth}
        \centering
        \small
        \setlength{\tabcolsep}{5pt}
        \begin{tabular}{lcccc}
        \multicolumn{5}{c}{\textit{Median speedup over baseline} ($\R^3$, $\R^5$)} \\
        \toprule
        Baseline & $d{=}1$ & $d{=}3$ & $d{=}5$ & $d{=}7$ \\
        \midrule
        SNOPT (disc.) & 15, 31 & 10, 15 & 14, 26 & 19, 28 \\
        SNOPT (exact) & 39, 159 & 440, N/A & N/A & N/A \\
        IPOPT (disc.) & 84, 156 & 33, 61 & 57, 126 & 34, 70 \\
        IPOPT (exact) & 14, 32 & 23, 30 & 71, 117 & 212, 397 \\
        \bottomrule
        \end{tabular}
        \caption{}
        \label{fig:speedup_table}
    \end{subfigure}

    \vspace{2.5em}

    \begin{subfigure}[b]{\linewidth}
        \centering
        \includegraphics[width=\linewidth]{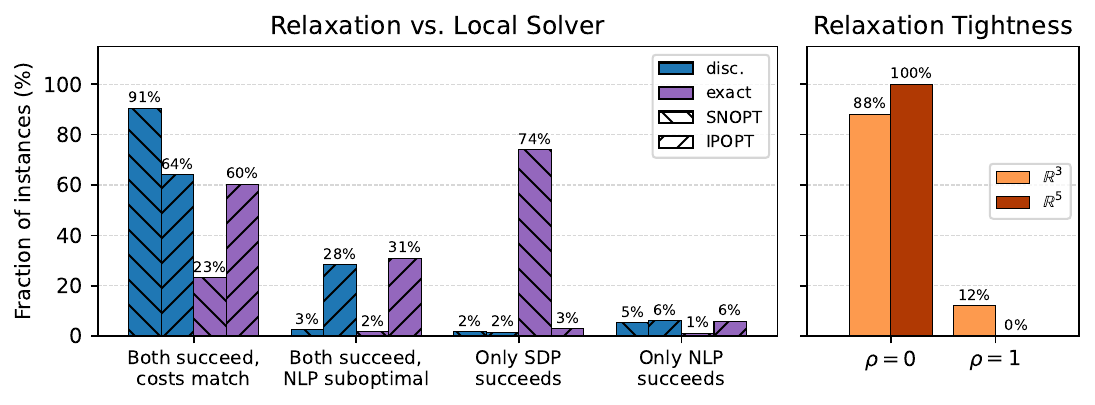}
        \caption{}
        \label{fig:sdr_vs_nlp_outcome}
    \end{subfigure}
\end{minipage}
\caption{
Empirical comparison of the relaxation against SNOPT and IPOPT across random problem instances (see~\cref{sec:computational_cost}).
(a) Solve times (log scale); grey bars indicate all instances timed out.
(b) Median speedup over each baseline, reported as $(\R^3, \R^5)$.
(c) Outcome breakdown against each baseline (left) and the distribution of the relaxation tightness (measured by $\rho = \operatorname{rank}(X - \Gamma^T \Gamma)$, see~\cref{sec:theoretical_results}) (right).
The relaxation is 1--2 orders of magnitude faster than the local solvers with 1--3 orders of magnitude smaller variance.
}
\label{fig:solve_times_comparison}
\end{figure*}

We empirically evaluate the tightness and computational cost of the relaxation.
We compare against multiple \acrshort*{nlp} baselines,
e.g. direct transcriptions of the same nonconvex problem solved with local nonlinear solvers,
and generally find that the relaxation is 1--2 orders of magnitude faster with 2--3 orders of magnitude lower variance in solve times.
Further,
when the strongest \acrshort*{nlp} baseline succeeds the relaxation is typically also tight, suggesting it captures primarily local information.

We compare against SNOPT~\cite{gillSNOPT2005} and IPOPT~\cite{waechterIPOPT2006} using two strategies for enforcing that each polynomial segment is collision-free.
The \textit{discretization approach} enforces~\cref{eq:nonneg_collision} at $20$ uniformly spaced values of $s \in [0,1]$ per segment, yielding a nonconvex quadratic constraint at each sample.
The \textit{exact approach} uses the exact condition in~\cref{eq:collision_free_conditions}.
For the latter approach we enforce $M \succeq 0$ through the factorization $M = LL^T$ for a lower triangular matrix $L$, as neither of the solvers support PSD constraints.

We compare across the problem configurations in~\Cref{tab:degree-configs}, matching the classical setup of~\cite{mellinger2011minimum}.
We evaluate the methods on 100 randomly generated instances in $\R^3$ (15 obstacles) and $\R^5$ (30 obstacles), with obstacle radii in $[0.3, 1.0]$ placed in the box $[-2, 2]^n$.
We apply the procedure in \cref{sec:facial_reduction} to similarly reduce the number of free variables for both the relaxation, SNOPT, and IPOPT.
We initialize SNOPT and IPOPT with a straight-line initial guess, reduce their feasibility tolerance to $10^{-4}$, and enforce a timeout of $30$ seconds.

\begin{table}[ht]
\centering
\begin{tabular}{ccccc}
\toprule
Degree & Cost & Continuity & BC & Segments \\
\midrule
1 & velocity ($k=1$) & $C^0$ & 0 & 12 \\
3 & acceleration ($k=2$) & $C^2$ & 2 & 6 \\
5 & jerk ($k=3$) & $C^3$ & 3 & 4 \\
7 & snap ($k=4$) & $C^4$ & 4 & 3 \\
\bottomrule
\end{tabular}
\caption{B\'{e}zier curve configurations. Cost minimizes $\int \|\gamma^{(k)}(t)\|^2\,dt$, polynomial degree is $2k-1$, continuity is $C^k$ ($C^0$ for $k=1$ to avoid overconstraining the problem), and BC order is the highest derivative fixed to zero at the endpoints.}
\label{tab:degree-configs}
\end{table}

\subsection{Computational Cost}
\label{sec:computational_cost}
We now compare solve times between the methods.
Against the discretization approach, the relaxation is 1--2 orders of magnitude faster across all polynomial degrees and both dimensions.
Against the exact approach, the gap is even larger: for $d \geq 3$, SNOPT times out on nearly all instances, and IPOPT solves with median solve times that grow steeply with the degree (reaching tens of seconds), while the relaxation solves with medians of 17--33\,ms in $\R^3$ and 24--63\,ms in $\R^5$.
The \acrshort*{sdp} relaxation also has a far smaller interquartile range: 1--4\,ms in $\R^3$ and 1--12\,ms in $\R^5$, 1--3 orders of magnitude smaller than the local solvers, whose interquartile ranges span 0.1--4.4\,s for the discretization approach and 0.07--6.9\,s for the exact approach (instances that reached the time limit excluded).

\subsection{Empirical Tightness}
\Cref{fig:tightness_geometry_2d} gives geometric intuition for when the relaxation is tight, showing the regions of tightness and looseness for a single start position across the workspace under different obstacle configurations.
Further, there is a strong correlation between when the relaxation is tight and when the strongest \acrshort*{nlp} baseline succeeds.
In \cref{fig:sdr_vs_nlp_outcome}, we show an outcome breakdown comparing the relaxation with each baseline, considering only instances where the relaxation is tight or the baseline succeeds.
Against the strongest baseline, both methods succeed with matching costs on 91\% of instances. On the remaining instances, either only one of the two methods succeeds, or both succeed but the relaxation reaches the global optimum while the baseline finds a suboptimal local solution.
Against the other baselines, this gap widens.
On a growing fraction of instances the relaxation finds the global optimum where the baseline returns only a local one, or succeeds where the baseline fails or times out.

By~\Cref{thm:tightness}, the relaxation is tight when extending the problem to higher dimensions does not reduce the cost, which empirically seems to approximately coincide with cases where local reasoning suffices to find the globally optimal path.
Further, we empirically observe values of $\rho = \operatorname{rank}(X - \Gamma^T \Gamma)$ of only $0$ (tight) or $1$ (not tight) and never higher, suggesting that the original nonconvex problems in both $\R^3$ and $\R^5$ do not get any easier after extending them by one extra dimension (in the sense defined in~\cref{sec:theoretical_results}, e.g. (P$_{\rho=1}$) might have lower cost than \eqref{eq:nonconvex_path_planning}, but the problems (P$_{\rho \geq 1}$) all have the same cost).

\section{Kinodynamic Motion Planning}
\label{sec:motion_planning_experiments}
\begin{figure*}[!t]
\centering
\includegraphics[width=\textwidth]{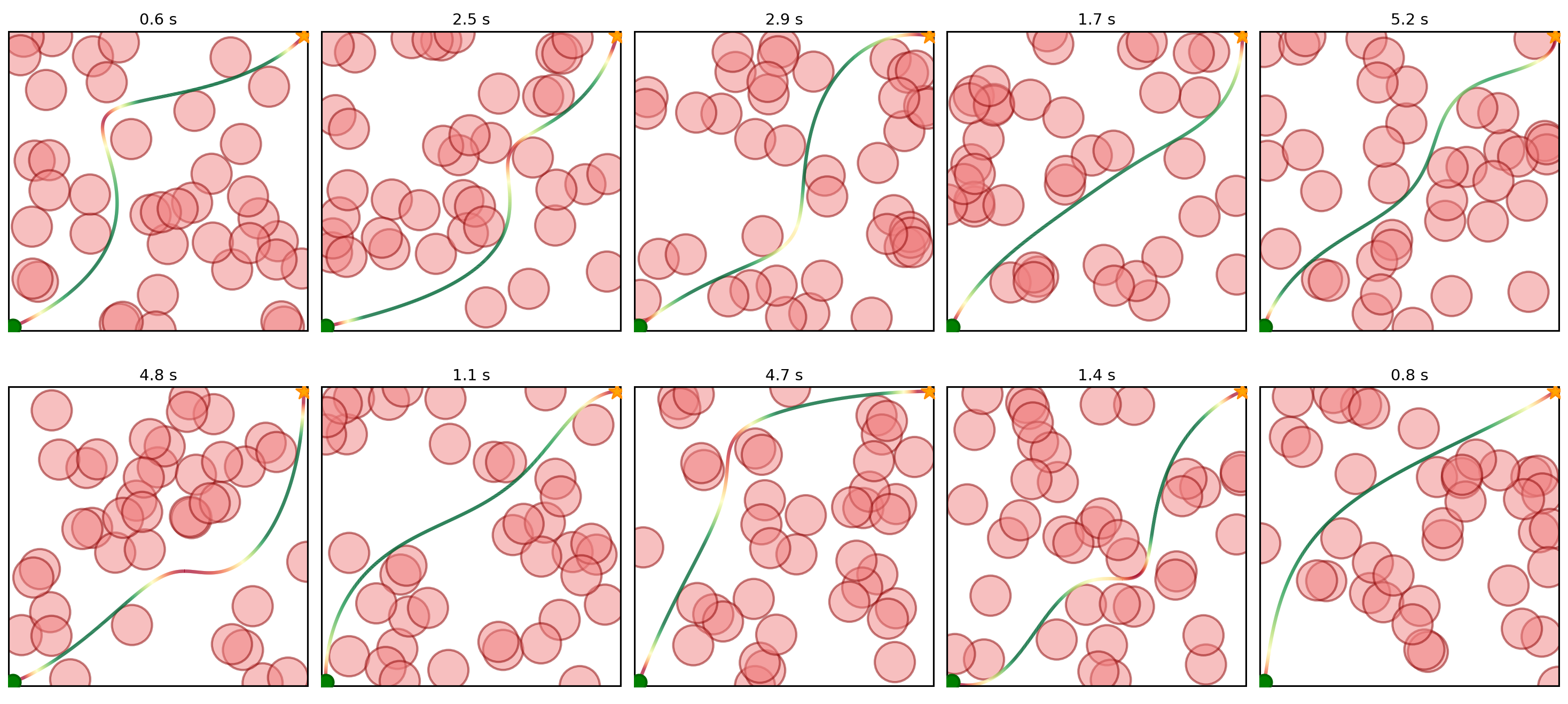}

\caption{Minimum-snap trajectories ($C^4$ continuous, degree~$9$ polynomials) found on 10 randomly generated $\R^2$ environments with 40 obstacles, by our RRT implementation that uses the \acrshort*{sdp} relaxation as the steering procedure in the \textsc{Extend} step and post-processes with a small convex program.
Each subplot title reports the total planning time.
Trajectories are colored by speed, from green (fast) to red (slow).}
\label{fig:rrt_sdp_steer_montage_2d}
\end{figure*}

We illustrate how the relaxation can serve as a convex subroutine within a higher-level planner.
Specifically, we use it as the local steering procedure in the \textsc{Extend} step of RRT~\cite{lavalle1998rapidly}, applied to minimum-snap quadrotor planning with degree~9, $C^4$-continuous trajectories.

We show that the semidefinite relaxation is well-suited as the steering function in an RRT.
By default, RRT uses straight-line steering that ignores the smoothness and dynamic constraints that motion planning typically requires.
A solution to this is to solve the two-point \acrfull*{bvp} analytically~\cite{webb2013kinodynamic, perez2012lqr}, but this applies only to certain systems and costs, and ignores state and input constraints.
Another approach plans collision-free waypoints and smooths them~\cite{richter2016polynomial}, or uses a local nonlinear solver as the steering function, but these inherit the brittleness of local solvers: reliance on an initial guess, high solve-time variance, and discretized collision avoidance.
On the other hand, the relaxation finds a collision-free, cost-optimal trajectory in a single convex solve, enforces collision avoidance exactly along the entire trajectory, requires no initial guess, has fast and consistent solve times, and naturally handles state and input constraints such as velocity and acceleration limits.

\subsection{RRT Implementation}
We compare three methods, which we refer to as \textit{SDP-RRT}, \textit{NLP-RRT}, and \textit{Geometric RRT}.
SDP-RRT uses the relaxation as the steering function, enforcing collision avoidance exactly along the entire trajectory.
NLP-RRT is an otherwise identical RRT that uses an optimized \acrshort*{nlp} trajectory optimizer as its steering function.
We enforce collision avoidance with Drake's~\cite{drake} \texttt{MinimumDistanceLowerBoundConstraint} (a smoothed signed-distance formulation), and apply several optimizations: inflate the obstacles to reduce the number of collision samples, loosen the tolerances, and switch between SNOPT and IPOPT depending on problem size.
Geometric RRT first uses straight-line geometric RRT to find collision-free waypoints, then smooths them with the same \acrshort*{nlp} trajectory optimizer.
We sample only in the space of positions, and thus for SDP-RRT and NLP-RRT we implement a partial-state RRT~\cite{zheng2021accelerating} that leaves the derivatives free.
These derivatives can grow rapidly as the tree deepens, and point-to-point steering functions cannot re-optimize across waypoints to control them, relying instead on heuristics such as terminal penalties~\cite{zheng2021accelerating}.
As a convex program, our SDP steering function instead optimizes directly over the free derivatives across multiple waypoints, removing the need for such heuristics.
In practice we re-optimize only the last two segments on each extension, which we find produces near-identical trajectories to re-optimizing the entire path back to the root, at much lower cost.
When connecting to the goal, the derivatives are fixed to enforce the boundary conditions.

\subsection{Trajectory Post-Processing}
After the RRT planners find a path, we post-process it to reduce cost.
For Geometric RRT and NLP-RRT, the path initializes a nonlinear program over the full trajectory.
For SDP-RRT, we use a cutting-plane procedure motivated by the geometric interpretation of the relaxation: since the relaxation finds the cheapest trajectory across a family of higher-dimensional problems, we add linear cuts that force this trajectory into the original subspace.
We sample points along the trajectory, add separating cuts between them and the obstacles, re-optimize, and repeat until the solution is collision-free, then resample to reduce cost further.
Because dense cuts leave the \acrshort*{lmi} obstacle constraints rarely active, we drop them, reducing each iteration to a convex \acrfull*{qp} that solves in a few milliseconds.
This is a heuristic post-processing step, and the method recently proposed in~\cite{werner2026biconvex} provides a more principled approach to a similar idea.

\subsection{Numerical Results}
\Cref{fig:hero} shows a collision-free quadrotor trajectory in $\R^3$.
We evaluate the full pipeline on 100 randomly generated $\R^2$ environments, each with 40 spherical obstacles of radius $1\,\text{m}$ in a $15 \times 15\,\text{m}$ square, comparing all three methods on pipeline success rate, total solve time (including refinement), steer time, and trajectory cost.
A run is successful if it produces a feasible, collision-free trajectory. For Geometric RRT this requires the \acrshort*{nlp} refinement to succeed, while \acrshort*{nlp}-RRT and \acrshort*{sdp}-RRT have a feasible trajectory once the RRT finds a path.
Steer time is the cumulative optimization time, excluding program construction overhead, which can be eliminated by pre-compiling the convex optimization programs with a tool such as CVXGEN~\cite{mattingley2012cvxgen}.

As shown in \cref{fig:rrt_pipeline_comparison}, with example trajectories in \cref{fig:rrt_sdp_steer_montage_2d}, \acrshort*{sdp}-RRT achieves the best performance across all metrics.
\acrshort*{nlp}-RRT shares our ability to reason about additional trajectory constraints, but its nonconvex steering is slow and brittle: it is over an order of magnitude slower, less reliable, and often falls back to the unrefined RRT path, yielding far higher-cost trajectories.
Geometric RRT, in contrast, is fast and reliable on this benchmark, which has no dynamic or input constraints, and \acrshort*{sdp}-RRT matches its solve time at a slightly higher success rate and comparable cost.
Crucially however, Geometric RRT decouples geometric planning from trajectory optimization and cannot incorporate constraints such as dynamics or input limits, whereas our convex steering can enforce additional convex constraints during exploration.

\begin{figure}[!t]
\centering
\includegraphics[width=\columnwidth]{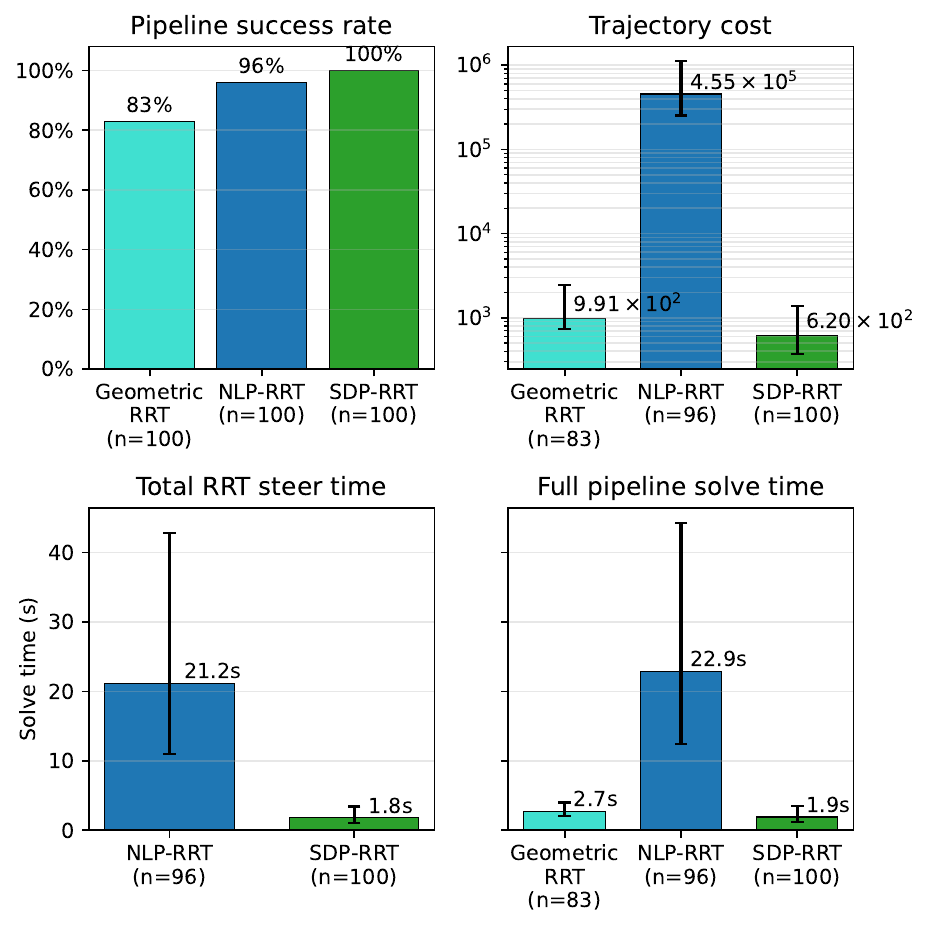}
\caption{Aggregate statistics showing median and interquartile ranges over 100 random environments for Geometric RRT, \acrshort*{nlp}-RRT, and \acrshort*{sdp}-RRT (ours), with costs and solve times reported over successful runs ($n$ per bar).
\acrshort*{sdp}-RRT substantially outperforms \acrshort*{nlp}-RRT,
and performs similar to Geometric RRT (with slightly higher success rate), which importantly, unlike our method, cannot incorporate additional constraints such as dynamics or input limits.}
\label{fig:rrt_pipeline_comparison}
\end{figure}

\section{Conclusion}
\label{sec:future_work}

While the relaxation is fast and reliable, our empirical analysis shows that it captures primarily local information and can be loose on instances where global reasoning is required.
A potentially interesting direction of research is to develop higher-order relaxations from the moment-SOS hierarchy~\cite{lasserre2001global,parrilo2003semidefinite} that can tighten the relaxation on such instances, and whether the symmetry reduction extends naturally to higher levels of the hierarchy.
Relatedly, since we never observe $\rho > 1$ in our experiments, an open question is to prove tighter bounds on $\rho$.
Another direction is to extend the framework and analysis to richer obstacle and robot geometries described by polynomials.
Further, since the relaxation is convex, it readily accommodates other convex constraints, extending the method to systems with linear dynamics and state or input constraints.
Finally, another promising direction is to develop a custom solver that exploits the highly structured nature of our formulation, which could potentially speed up the relaxation further.

\appendix
\label{sec:appendix}

\subsection{Connection To the Shor Relaxation}
\label{sec:symmetry}

In this section, we show that our relaxation~\eqref{eq:first_order_relaxation_single} is mathematically equivalent to the \textit{Shor relaxation}~\cite{shor1987quadratic}, i.e. the first level of the moment hierarchy, of~\eqref{eq:nonconvex_path_planning}.
The Shor relaxation lifts the full vector of decision variables $y \in \R^{n(d+1)}$ to a symmetric matrix $Y \in \mathbb{S}^{n(d+1)}$ that replaces the outer product $yy^T$, rendering each (potentially nonconvex) quadratic term linear, $y^T H y = \iprod{H, Y}$, and relaxes the nonconvex lifting equality $Y = yy^T$ to the convex constraint $Y \succeq yy^T$.
Our relaxation~\eqref{eq:first_order_relaxation_single} follows the same recipe, but instead lifts to the smaller Gram matrix $X = \Gamma^T \Gamma \in \mathbb{S}^{d+1}$.
Specifically, we show that \eqref{eq:first_order_relaxation_single} is a symmetry-reduced version of the Shor relaxation, where we exploit an $O(n)$ symmetry inherent to~\eqref{eq:nonconvex_path_planning} to reformulate the problem over a \acrshort*{psd} cone that is smaller by a factor of $n$, the ambient dimension.
We derive the Shor relaxation in vectorized form to match the existing literature, but note that since~\eqref{eq:nonconvex_path_planning} is a matrix problem, one could equivalently derive it directly in matrix form.
All proofs are in~\cref{sec:symmetry_proofs}.

The first step to deriving the Shor relaxation of~\eqref{eq:nonconvex_path_planning} is to rewrite~\eqref{eq:nonconvex_path_planning} in a vectorized form.
Define
$y = \mathrm{vec}(\Gamma) = 
(\gamma_0, \ldots, \gamma_d)
\in \R^{n(d+1)}$
as the vector of vertically stacked polynomial coefficients.
Using the Kronecker product $\otimes$, we rewrite the linear constraints directly in terms of $y$.
Applying the vectorization identity for Kronecker products,
$\mathrm{vec}(LMN) = (N^T \otimes L)\mathrm{vec}(M)$,
the boundary conditions $\Gamma A = B$ become $(A^T \otimes I_n)y = \mathrm{vec}(B)$.

Next, we rewrite the cost and the arguments to $\mathcal{L}$ in terms of $yy^T$ and $y$.
Define the \textit{partial trace} $\mathrm{tr}_n: \mathbb{S}^{n(d+1)} \rightarrow \mathbb{S}^{d+1}$ by $[\mathrm{tr}_n(M)]_{ij} = \mathrm{tr}(M_{ij})$, where $M_{ij} \in \R^{n \times n}$ is the $(i,j)$ block of $M$.
The partial trace will let us translate between the outer products that arise from vectorization and the Gram matrices that appear in \eqref{eq:nonconvex_path_planning}.
Specifically, a direct computation shows that for any $M, N \in \R^{n \times (d+1)}$,
\begin{align}
M^T N = \mathrm{tr}_n(\mathrm{vec}(M)\mathrm{vec}(N)^T).
\label{eq:partial_trace_identity}
\end{align}
Applying~\eqref{eq:partial_trace_identity} to the arguments of $\mathcal{L}$ in~\eqref{eq:nonconvex_path_planning} gives all the substitutions we need to vectorize the problem. Setting $(M, N)$ to $(\Gamma, \Gamma)$, $(\Gamma, cu^T)$, and $(cu^T, cu^T)$ yields $\Gamma^T \Gamma = \mathrm{tr}_n(yy^T)$, $\Gamma^T cu^T = \mathrm{tr}_n(y(u\otimes c)^T)$, and $(c^Tc)\,uu^T = \mathrm{tr}_n((uu^T) \otimes (cc^T))$, respectively, where we used $\mathrm{vec}(cu^T) = u \otimes c$ and the Kronecker mixed-product rule.
The vectorized form of \eqref{eq:nonconvex_path_planning} is then
\begin{align*}
    \min \quad & \mathrm{tr}(G^{(k)} \, \mathrm{tr}_n(yy^T)) \\
    \subto \quad
    & \mathcal{L}(\mathrm{tr}_n(yy^T - (u \otimes c)\, y^T - y\, (u \otimes c)^T \\
    & \quad\quad + (uu^T) \otimes (cc^T)), Q_0, Q_1) = 0, \\
    & Q_0 \succeq 0, \; Q_1 \succeq 0, \\
    & (A^T \otimes I_n)\, y = \mathrm{vec}(B),
\end{align*}
with $y \in \R^{n(d+1)}$ and $Q_0, Q_1$ unchanged from~\eqref{eq:nonconvex_path_planning}.

The Shor relaxation is then formed by introducing $Y \in \mathbb{S}^{n(d+1)}$ with the constraint $Y = yy^T$, and relaxing this nonconvex equality into $Y \succeq yy^T$.
The resulting relaxation is:
\begin{align*}
\tag{Shor}
\label{eq:shor}
    \min \quad &\mathrm{tr}(G^{(k)} \, \mathrm{tr}_n(Y))  \\
    \subto \quad
    & \mathcal{L}(\mathrm{tr}_n(
    Y - (u \otimes c) y^T - y (u \otimes c)^T \\
    & \quad\quad + (uu^T) \otimes (cc^T)), Q_0, Q_1) = 0, \\
    & Q_0 \succeq 0, \; Q_1 \succeq 0, \\
    &(A^T \otimes I_n)y = \mathrm{vec}(B), \\
    &(A^T \otimes I_n)Y = \mathrm{vec}(B)\, y^T, \\
    &\bmat{
       1 & y^T  \\
       y & Y
    } \succeq 0.
\end{align*}
We have included the implied constraint $(A^T \otimes I_n)Y = \mathrm{vec}(B)\, y^T$, which is redundant when $Y = yy^T$ but tightens the relaxation, analogously to $XA = \Gamma^T B$ in~\eqref{eq:first_order_relaxation_single}.
The \acrshort*{psd} variables of the two relaxations are related by $\mathrm{tr}_n(Y) = X$.

We now show how \eqref{eq:first_order_relaxation_single} is a symmetry-reduced version of (Shor).
To expose the symmetry, we decompose $Y = yy^T + S$, where $S \succeq 0$ is a positive semidefinite slack variable.
Since $(A^T \otimes I_n)y = \mathrm{vec}(B)$, the implied constraint reduces to $(A^T \otimes I_n)S = 0$.
With this substitution, (Shor) becomes
\begin{align*}
\tag{Shor$'$}
\label{eq:shor_prime}
    \min \quad & \mathrm{tr}(G^{(k)} \, \mathrm{tr}_n(yy^T + S)) \\
    \subto \quad
    & \mathcal{L}(\mathrm{tr}_n(
    yy^T + S - (u \otimes c) y^T \\
    & \quad\quad - y (u \otimes c)^T + (uu^T) \otimes (cc^T)), \\
    & \quad\quad Q_0, Q_1) = 0, \\
    & Q_0 \succeq 0, \; Q_1 \succeq 0, \\
    & (A^T \otimes I_n)y = \mathrm{vec}(B), \\
    & (A^T \otimes I_n) S = 0, \\
    & S \succeq 0.
\end{align*}
The key observation to see the symmetry in the problem is that both the cost and the linear map constraint depend on $S$ only through $\mathrm{tr}_n(S) \in \mathbb{S}^{d+1}$, a much smaller matrix than $S \in \mathbb{S}^{n(d+1)}$.
Let $O(n) := \{Q \in \R^{n \times n} \mid Q^T Q = I_n\}$ denote the orthogonal group.

\begin{lemma}
\label{lem:shor_invariance}
\textnormal{(Shor$'$)} is invariant under the transformation
\begin{align}
    T_Q(S) = (I_{d+1} \otimes Q^T) \, S \, (I_{d+1} \otimes Q)
    \label{eq:symmetry_transform}
\end{align}
for any $Q \in O(n)$.
That is, if $(y, S)$ is feasible, then so is $(y, T_Q(S))$ with the same cost.
\end{lemma}

Since (Shor$'$) is invariant under the action of $O(n)$ on $S$, and the feasible set is convex in $S$ for any fixed $y$, we can restrict $S$ to the
\textit{fixed-point subspace} defined as
$\{S : T_Q(S) = S \text{ for all } Q \in O(n)\}$
without loss of optimality.
This follows from the standard averaging argument for symmetry reduction of semidefinite programs; see,
e.g.,~\cite{gatermannParrilo2004} for details.
  
\begin{lemma}
\label{lem:fixed_point_subspace}
The fixed-point subspace of the transformation~\eqref{eq:symmetry_transform} is
\begin{align}
    \bar{S} = \frac{1}{n} \Lambda \otimes I_n, \quad \Lambda \in \mathbb{S}^{d+1},
    \label{eq:fixed_point_subspace}
\end{align}
i.e., the set of matrices satisfying 
$T_Q(S) = S$ for all $Q \in O(n)$.
\end{lemma}

\begin{theorem}
\label{thm:sdp_is_symmetry_reduced}
\eqref{eq:first_order_relaxation_single} is the symmetry-reduced version of \textnormal{(Shor)}.
Specifically, restricting $S$ in \textnormal{(Shor$'$)} to the fixed-point subspace~\eqref{eq:fixed_point_subspace} recovers~\eqref{eq:first_order_relaxation_single}.
\end{theorem}

The symmetry reduction thus replaces the PSD variable $S \in \mathbb{S}^{n(d+1)}$ with $\Lambda \in \mathbb{S}^{d+1}$, reducing the size of the PSD cone by a factor of $n$.

\subsection{Proofs for the Theoretical Results}
\label{sec:proofs_theoretical_results}

In the proofs below, $Q_0$ and $Q_1$ are carried over unchanged between the programs, so we omit them from feasible points.

\begin{proof}[Proof of \Cref{thm:lifted_monotonicity}]
Start by noticing that the cost in \eqref{eq:nonconvex_path_planning_lifted} is unchanged by $\rho$, and thus all \eqref{eq:nonconvex_path_planning_lifted} share the same cost (to see this, consider introducing $X = \Gamma^T \Gamma + V^T V$).
Let
$\mathcal{F}_{\text{P}_\rho}$
be the feasible set of \eqref{eq:nonconvex_path_planning_lifted}, and let
$\Pi_\rho$ be the operator that projects a set onto the variables from \eqref{eq:nonconvex_path_planning_lifted}.
As the cost is shared across \eqref{eq:nonconvex_path_planning_lifted}, it is enough to show
$\mathcal{F}_{\text{P}_\rho}
\subseteq
\Pi_\rho(
\mathcal{F}_{\text{P}_{\rho+1}}
)
$
to show that the cost is monotonically decreasing as we increase $\rho$.
Let
$(\Gamma', V')$ be any feasible point to \eqref{eq:nonconvex_path_planning_lifted}, where $V' \in \R^{\rho \times (d+1)}$.
Then $(\Gamma', (V', 0))$ is feasible for $(\text{P}_{\rho+1})$, where $(V', 0) \in \R^{(\rho+1) \times (d+1)}$.
\end{proof}

\begin{proof}[Proof of \Cref{thm:lower_bound}]
Start by defining a semidefinite relaxation for the lifted program \eqref{eq:nonconvex_path_planning_lifted} for each $\rho \in \N$ (by following the same procedure in \cref{sec:first_order_relaxation}), but with $X = \Gamma^T \Gamma + V^T V$, as:
\begin{align*}
\tag{SDP$_\rho$}
\label{eq:sdpr}
    \min \quad &\mathrm{tr}(G^{(k)} \, X) \\
    \subto
    \quad& \mathcal{L}(X - (\Gamma^T c) u^T - u (\Gamma^T c)^T \\
    &\qquad + (c^T c) uu^T, Q_0, Q_1) = 0, \\
    \quad & Q_0 \succeq 0, \; Q_1 \succeq 0, \\
    & \Gamma A = B, \quad V A = 0 \\
    & X \succeq \Gamma^T \Gamma + V^T V.
\end{align*}
The collision avoidance constraint in~(SDP$_\rho$) is identical to the one in~\eqref{eq:first_order_relaxation_single}, and the last constraint can equivalently be expressed as an LMI via the Schur complement, as in~\eqref{eq:schur_lmi}.
Because (SDP$_\rho$) is a convex relaxation of \eqref{eq:nonconvex_path_planning_lifted} we immediately have
$c_{\text{SDP}_\rho}^* \leq
c_{\text{P}_\rho}^*$.
It remains to show
$c_{\text{SDP}_\rho}^* =
c_{\text{SDP}}^*$ for all $\rho$.
Let
$\mathcal{F}_{\text{SDP}_\rho}$
be the feasible set of (SDP$_\rho$)
and
$
\Pi(
\mathcal{F}_{\text{SDP}_{\rho}}
)
$
its projection onto the variables from \eqref{eq:first_order_relaxation_single}.
The cost for \eqref{eq:first_order_relaxation_single} and (SDP$_\rho$) is the same, so we will show that
$\mathcal{F}_{\text{SDP}}
=
\Pi(
\mathcal{F}_{\text{SDP}_{\rho}}
)
$.
First, let $(\Gamma', X')$ be any feasible point to \eqref{eq:first_order_relaxation_single}.
Then $(\Gamma', V' = 0, X')$ is feasible to (SDP$_\rho$):
the collision avoidance, tightening, and boundary constraints for $\Gamma$ carry over directly, $V' A = 0$ is trivially satisfied, and $X' \succeq (\Gamma')^T \Gamma' = (\Gamma')^T \Gamma' + 0$ satisfies the last constraint in~(SDP$_\rho$).
For the other direction,
let $(\Gamma', V', X')$ be feasible to (SDP$_\rho$), where $\Gamma' \in \R^{n \times (d+1)}$ and $V' \in \R^{\rho \times (d+1)}$.
We show that $(\Gamma', X')$ is feasible to \eqref{eq:first_order_relaxation_single}.
The boundary constraint $\Gamma' A = B$ and tightening constraint $X' A = (\Gamma')^T B$ carry over directly from~(SDP$_\rho$).
Finally, from the Schur complement of the LMI in~(SDP$_\rho$), we have
$X' \succeq (\Gamma')^T \Gamma' + (V')^T V' \succeq (\Gamma')^T \Gamma'$,
which satisfies the constraint in~\eqref{eq:first_order_relaxation_single}.
\end{proof}

\begin{proof}[Proof of \Cref{thm:geometric_interpretation}]
    Let $(\Gamma^*, X^*)$ be an optimal solution to~\eqref{eq:first_order_relaxation_single}, and let $\rho = \text{rank}(X^* - (\Gamma^*)^T\Gamma^*)$.
    First, consider the case when $\rho = 0$, i.e.\ $X^* = (\Gamma^*)^T \Gamma^*$.
    In this case, $\Gamma^*$ is immediately feasible for $(\text{P})$,
    and since
    $c_\text{SDP}^*$ provides a lower bound that is attained by this feasible point, it must be optimal.
    Next, consider the case when $\rho > 0$.
    From \eqref{eq:schur_lmi} we have
    $X^* - (\Gamma^*)^T \Gamma^* \succeq 0$,
    so we can factorize
    \begin{align*}
        X^* - (\Gamma^*)^T \Gamma^* = (V^*)^T V^*
    \end{align*}
    where $V^* \in \R^{\rho \times (d+1)}$, $V^* \neq 0$.
    By construction,
    $(\Gamma^*, V^*)^T (\Gamma^*, V^*) = (\Gamma^*)^T \Gamma^* + (V^*)^T V^* = X^*$.
    We now show that
    $(\Gamma^*, V^*)$ is optimal for $(\text{P}_\rho)$.
    First, we show feasibility.
    The collision avoidance constraint in $(\text{P}_\rho)$ is satisfied because
    $(\Gamma^*, V^*)^T (\Gamma^*, V^*) = X^*$
    so the constraint reduces to the one in~\eqref{eq:first_order_relaxation_single}.
    For the boundary constraint, we have from \eqref{eq:first_order_relaxation_single} that $\Gamma^* A = B$ and $X^*A = (\Gamma^*)^T B$, which gives
    \begin{align*}
        (\Gamma^*)^T B &= X^*A \\
        &= (\Gamma^*)^T \Gamma^* A + (V^*)^T V^* A \\
        &= (\Gamma^*)^T B + (V^*)^T V^* A,
    \end{align*}
    so $(V^*)^T V^* A = 0$, implying $V^* A = 0$.
    This shows that the tightening constraints $X^* A = (\Gamma^*)^T B$ in~\eqref{eq:first_order_relaxation_single} are necessary to ensure that $V^*$ satisfies the boundary conditions $V^* A = 0$ in~\eqref{eq:nonconvex_path_planning_lifted}.
    Finally, we show optimality.
    By \Cref{thm:lower_bound}, we have
    $c_\text{SDP}^* \leq c_{\text{P}_\rho}^*$.
    By construction,
    $c_\text{SDP}^* = \mathrm{tr}(G^{(k)} X^*) = \mathrm{tr}(G^{(k)} (
    (\Gamma^*)^T\Gamma^* + (V^*)^T V^*))$.
    As $(\Gamma^*, V^*)$ is feasible for $(\text{P}_\rho)$ and attains the lower bound, it must be optimal.
\end{proof}

\begin{proof}[Proof of \Cref{lem:rho_bound}]
Let $M = X - \Gamma^T\Gamma$ for an optimal $(\Gamma, X)$, so $\rho = \operatorname{rank}(M)$. By the boundary constraint $\Gamma A = B$ and the tightening constraint $X A = \Gamma^T B$, we have $MA = XA - \Gamma^T\Gamma A = 0$, so $\operatorname{col}(A) \subseteq \ker M$ and $\rho \le (d+1) - \operatorname{rank}(A) = (d+1) - 2(\ell+1)$, as $\operatorname{rank}(A) = 2(\ell+1)$ since the boundary conditions are independent Hermite interpolation conditions~\cite{stoerNumericalAnalysis}.
\end{proof}

\begin{proof}[Proof of \Cref{thm:tightness}]
    We start by showing necessity.
    Suppose
    $c_\text{P}^* =
    c_{\text{P}_\rho}^*
    $ for all $\rho \in \N$.
    Then \Cref{thm:geometric_interpretation} immediately gives us a $\rho$ such that
    $
    c_\text{SDP}^* =
    c_{\text{P}_\rho}^* =
    c_\text{P}^*$.
    Next, we show sufficiency.
    Suppose
    $c_\text{SDP}^* =
    c_{\text{P}}^*$.
    \Cref{thm:lower_bound} gives
    $
    c_{\text{P}}^* =
    c_\text{SDP}^* \leq
    c_{\text{P}_\rho}^*$ for all $\rho \in \N$.
    The reverse inequality follows immediately from \Cref{thm:lifted_monotonicity}, and equality must hold.
\end{proof}

\begin{proof}[Proof of \Cref{thm:uniqueness}]
    Let $(\Gamma, X)$ be any optimal solution to~\eqref{eq:first_order_relaxation_single}. We consider two cases.
    First, suppose $X = \Gamma^T\Gamma$.
    Then $\Gamma$ is feasible for~\eqref{eq:nonconvex_path_planning} with cost
    $\mathrm{tr}(G^{(k)} \Gamma^T\Gamma) = c_\text{SDP}^* = c_\text{P}^*$,
    so $\Gamma$ achieves the lower bound of the relaxation and must therefore also be optimal for~\eqref{eq:nonconvex_path_planning}.
    By the uniqueness assumption (with $\rho = 0$),
    $\Gamma = \Gamma^*$, and therefore $X = (\Gamma^*)^T\Gamma^*$.
    Second, suppose $X \neq \Gamma^T\Gamma$.
    Let $\rho = \operatorname{rank}(X - \Gamma^T\Gamma) > 0$
    and factor $V^TV = X - \Gamma^T\Gamma$, where $V \neq 0$ due to $\rho > 0$.
    By \Cref{thm:geometric_interpretation}, $(\Gamma, V)$ is optimal for~\eqref{eq:nonconvex_path_planning_lifted}.
    By the uniqueness assumption, $(\Gamma, V) = (\Gamma^*, 0)$,
    which contradicts $V \neq 0$.
    Therefore, every optimal solution to~\eqref{eq:first_order_relaxation_single} equals
    $(\Gamma^*, (\Gamma^*)^T\Gamma^*)$.
\end{proof}

\subsection{Proofs for Symmetry Reduction}
\label{sec:symmetry_proofs}
\begin{proof}[Proof of \Cref{lem:shor_invariance}]
Let $\tilde{S} = (I_{d+1} \otimes Q^T) S (I_{d+1} \otimes Q)$.
Writing $S$ in terms of its $n \times n$ blocks $S_{ij}$, the transformation acts as $S_{ij} \mapsto Q^T S_{ij} Q$.
By the cyclic property of the trace, $\mathrm{tr}(Q^T S_{ij} Q) = \mathrm{tr}(S_{ij})$, so the partial trace is invariant under this transformation, $\mathrm{tr}_n(\tilde{S}) = \mathrm{tr}_n(S)$, and the cost and linear map constraint are preserved.
For the boundary constraint, the mixed-product property of Kronecker products gives
$(A^T \otimes I_n)(I_{d+1} \otimes Q^T) = A^T \otimes Q^T = (I \otimes Q^T)(A^T \otimes I_n)$
, so $(A^T \otimes I_n) \tilde{S} = (I_p \otimes Q^T)(A^T \otimes I_n)S(I_{d+1} \otimes Q) = 0$
(by the implied boundary constraint).
Finally, $\tilde{S} \succeq 0$ since it is a congruence of $S \succeq 0$.
\end{proof}

\begin{proof}[Proof of \Cref{lem:fixed_point_subspace}]
The transformation acts independently on each $n \times n$ block as $\bar{S}_{ij} \mapsto Q^T \bar{S}_{ij} Q$, so invariance requires $Q^T \bar{S}_{ij} Q =
\bar{S}_{ij}$ for all $Q \in O(n)$.
It is a standard result that this implies $\bar{S}_{ij} = (1/n)\lambda_{ij} I_n$ for some $\lambda_{ij} \in \R$.
Collecting these scalars into $\Lambda = [\lambda_{ij}] \in \mathbb{S}^{d+1}$ gives $\bar{S} = \frac{1}{n}\Lambda \otimes I_n$.
\end{proof}

\begin{proof}[Proof of \Cref{thm:sdp_is_symmetry_reduced}]
We substitute $S = \frac{1}{n} \Lambda \otimes I_n$ into (Shor$'$) and define $X = \Gamma^T\Gamma + \Lambda$.
Since $\mathrm{tr}_n(S) = \Lambda$ and $\mathrm{tr}_n(yy^T) = \Gamma^T\Gamma$, we have
$\mathrm{tr}_n(yy^T + S) = X$,
recovering the cost and linear map constraint of~\eqref{eq:first_order_relaxation_single}.
The boundary condition $\Gamma A = B$ is unchanged, and $(A^T \otimes I_n)S = 0$ becomes $\Lambda A = 0$, so $XA = (\Gamma^T\Gamma + \Lambda)A = \Gamma^T B$.
Finally,
$S \succeq 0$ if and only if $\Lambda \succeq 0$, since the eigenvalues of $\Lambda \otimes I_n$ are exactly those of $\Lambda$, each with multiplicity $n$, which is equivalent to $X \succeq \Gamma^T\Gamma$, i.e., $\bmat{I_n & \Gamma \\ \Gamma^T & X} \succeq 0$.
\end{proof}

\subsection{Bézier Curves and the Bernstein Basis}
\label{sec:bezier_curves}
In our implementation, we use the standard Bernstein basis, which differs from the scaled Bernstein basis used in the main text by a binomial scaling.
The $i$-th standard Bernstein basis polynomial of degree $d$ is defined as
\begin{align}
    B_i^d(s) := \binom{d}{i}s^i(1-s)^{d-i},
    \label{eq:bernstein_poly_def}
\end{align}
for $i = 0, \ldots, d$ and $s \in [0,1]$.
We use the following identities, which follow from the basis definition:
\begin{align}
    s(1-s) B^d_i(s) &= \frac{(d+1-i)(i+1)}{(d+1)(d+2)} B^{d+2}_{i+1}(s),
    \label{eq:bernstein_s_sm1_rule} \\
    B_i^d(s) B_j^d(s) &= \frac{\binom{d}{i} \binom{d}{j}}{\binom{2d}{i+j}} B^{2d}_{i+j}(s).
    \label{eq:bernstein_product_rule}
\end{align}

\subsubsection{Curve Costs for Bézier Curves}
\label{sec:gram_matrix}

Derivatives of Bézier curves are Bézier curves of reduced degree:
\begin{align}
    \gamma^{(k)}(s) = \frac{d!}{(d-k)!} \, \Gamma D_d^{(k)} b_{d-k}(s),
    \label{eq:bezier_derivative}
\end{align}
where $D_d^{(k)} := D_d D_{d-1} \cdots D_{d-k+1} \in \R^{(d+1) \times (d-k+1)}$ and $D_d \in \R^{(d+1) \times d}$ is the forward difference matrix with $-1$ on the diagonal and $1$ on the subdiagonal.
The Gram matrix $G_d$ from ~\eqref{eq:gram_matrix_def} has entries
\begin{align}
    [G_d]_{ij} = \frac{1}{2d+1} \frac{\binom{d}{i}\binom{d}{j}}{\binom{2d}{i+j}},
    \label{eq:gram_matrix_bezier}
\end{align}
which follows from the product rule~\eqref{eq:bernstein_product_rule} and $\int_0^1 B_i^{d}(s) \, ds = 1/(d+1)$.
Substituting~\eqref{eq:bezier_derivative} into the cost integral gives $\int_0^1 \|\gamma^{(k)}(s)\|_2^2 \, ds = \mathrm{tr}(G^{(k)} \Gamma^T \Gamma)$, where
\begin{align}
    G^{(k)} := \left(\frac{d!}{(d-k)!}\right)^2 D_d^{(k)} \, G_{d-k} \, (D_d^{(k)})^T.
    \label{eq:gram_matrix_kth_derivative}
\end{align}

\subsubsection{Collision Avoidance for Bézier Curves}
\label{sec:bernstein_ids}
We now write the collision-avoidance condition from \Cref{lem:collision_avoidance_bernstein} for Bézier curves, i.e., in the standard Bernstein basis. The derivation follows exactly as in the main body, with the standard Bernstein basis replacing the scaled one.

The partition of unity property of the standard Bernstein basis immediately gives $u = e$:
\begin{align}
    \sum_{i=0}^d B_i^d(s) = e^T b_d(s) = 1.
    \label{eq:bezier_sum_one}
\end{align}
From the product rule~\eqref{eq:bernstein_product_rule}, the map $\mathcal{S}_d: \mathbb{S}^{d+1} \rightarrow \R^{2d+1}$ is given by:
\begin{align}
    [\mathcal{S}_d(M)]_k = \sum_{\substack{i+j=k}}
    \frac{\binom{d}{i} \binom{d}{j}}{\binom{2d}{k}} M_{ij}.
\end{align}

Similarly, using \eqref{eq:bernstein_s_sm1_rule}, $\mathcal{T}_d: \mathbb{S}^d \rightarrow \R^{2d+1}$ is given by:
\begin{align}
    [\mathcal{T}_d(M)]_k = \begin{cases}
        \frac{k(2d-k)}{2d(2d-1)} [\mathcal{S}_{d-1}(M)]_{k-1} & \text{if } 1 \leq k \leq 2d-1, \\
        0 & \text{otherwise}.
    \end{cases}
\end{align}

\section*{Acknowledgement}
\label{sec:acknowledgement}

This material is based upon work supported by the National Science Foundation under ``Center: NSF Engineering Research Center for Human AugmentatioN via Dexterity (HAND)''.
This work was also supported by Amazon Inc., under Awards \#2D-19158851, \#2D-19158853, and \#2D-15694048, and by the Aker Scholarship.
Any opinions, findings, conclusions, or recommendations expressed in this material are those of the authors and do not necessarily reflect the views of the funding agencies.

\bibliographystyle{IEEEtran}
\bibliography{IEEEabrv,ref}

\begin{thebibliography}{10}
\providecommand{\url}[1]{#1}
\csname url@rmstyle\endcsname
\providecommand{\newblock}{\relax}
\providecommand{\bibinfo}[2]{#2}
\providecommand\BIBentrySTDinterwordspacing{\spaceskip=0pt\relax}
\providecommand\BIBentryALTinterwordstretchfactor{4}
\providecommand\BIBentryALTinterwordspacing{\spaceskip=\fontdimen2\font plus
\BIBentryALTinterwordstretchfactor\fontdimen3\font minus
  \fontdimen4\font\relax}
\providecommand\BIBforeignlanguage[2]{{%
\expandafter\ifx\csname l@#1\endcsname\relax
\typeout{** WARNING: IEEEtran.bst: No hyphenation pattern has been}%
\typeout{** loaded for the language `#1'. Using the pattern for}%
\typeout{** the default language instead.}%
\else
\language=\csname l@#1\endcsname
\fi
#2}}

\bibitem{chang2005shortest}
E.-C. Chang, S.~W. Choi, D.~Kwon, H.~Park, and C.~K. Yap, ``Shortest path
  amidst disc obstacles is computable,'' in \emph{Proceedings of the
  Twenty-First Annual Symposium on Computational Geometry}, 2005, pp. 116--125.

\bibitem{canny1987new}
J.~Canny and J.~Reif, ``New lower bound techniques for robot motion planning
  problems,'' in \emph{28th Annual Symposium on Foundations of Computer Science
  ({SFCS} 1987)}.\hskip 1em plus 0.5em minus 0.4em\relax IEEE, 1987, pp.
  49--60.

\bibitem{toth2017handbook}
C.~D. Toth, J.~O'Rourke, and J.~E. Goodman, \emph{Handbook of discrete and
  computational geometry}.\hskip 1em plus 0.5em minus 0.4em\relax CRC Press,
  2017.

\bibitem{hargraves1987direct}
C.~R. Hargraves and S.~W. Paris, ``Direct trajectory optimization using
  nonlinear programming and collocation,'' \emph{Journal of Guidance, Control,
  and Dynamics}, vol.~10, no.~4, pp. 338--342, 1987.

\bibitem{zucker2013chomp}
M.~Zucker, N.~Ratliff, A.~D. Dragan, M.~Pivtoraiko, M.~Klingensmith, C.~M.
  Dellin, J.~A. Bagnell, and S.~S. Srinivasa, ``{CHOMP}: Covariant
  {Hamiltonian} optimization for motion planning,'' \emph{The International
  Journal of Robotics Research}, vol.~32, no. 9-10, pp. 1164--1193, 2013.

\bibitem{tracy2023differentiable}
K.~Tracy, T.~A. Howell, and Z.~Manchester, ``Differentiable collision detection
  for a set of convex primitives,'' in \emph{2023 {IEEE} International
  Conference on Robotics and Automation ({ICRA})}.\hskip 1em plus 0.5em minus
  0.4em\relax IEEE, 2023, pp. 3663--3670.

\bibitem{dai2023certifiedpolyhedraldecompositionscollisionfree}
H.~Dai, A.~Amice, P.~Werner, A.~Zhang, and R.~Tedrake, ``Certified polyhedral
  decompositions of collision-free configuration space,'' \emph{The
  International Journal of Robotics Research}, vol.~43, no.~9, pp. 1322--1341,
  2024.

\bibitem{werner2024faster}
P.~Werner, T.~Cohn, R.~H. Jiang, T.~Seyde, M.~Simchowitz, R.~Tedrake, and
  D.~Rus, ``Faster algorithms for growing collision-free convex polytopes in
  robot configuration space,'' \emph{arXiv preprint arXiv:2410.12649}, 2024.

\bibitem{deits2015computing}
R.~Deits and R.~Tedrake, ``Computing large convex regions of obstacle-free
  space through semidefinite programming,'' in \emph{Algorithmic Foundations of
  Robotics {XI}: Selected Contributions of the Eleventh International Workshop
  on the Algorithmic Foundations of Robotics}.\hskip 1em plus 0.5em minus
  0.4em\relax Springer, 2015, pp. 109--124.

\bibitem{deits2015efficient}
------, ``Efficient mixed-integer planning for {UAVs} in cluttered
  environments,'' in \emph{2015 {IEEE} International Conference on Robotics and
  Automation ({ICRA})}.\hskip 1em plus 0.5em minus 0.4em\relax IEEE, 2015, pp.
  42--49.

\bibitem{marcucci2023motion}
T.~Marcucci, M.~Petersen, D.~Von~Wrangel, and R.~Tedrake, ``Motion planning
  around obstacles with convex optimization,'' \emph{Science Robotics}, vol.~8,
  no.~84, p. eadf7843, 2023.

\bibitem{marcucci2025biconvex}
T.~Marcucci, M.~Halm, W.~Yang, D.~Lee, and A.~D. Marchese, ``A biconvex method
  for minimum-time motion planning through sequences of convex sets,''
  \emph{arXiv preprint arXiv:2504.18978}, 2025.

\bibitem{werner2025superfast}
P.~Werner, R.~Cheng, T.~Stewart, R.~Tedrake, and D.~Rus, ``Superfast
  configuration-space convex set computation on {GPUs} for online motion
  planning,'' \emph{arXiv preprint arXiv:2504.10783}, 2025.

\bibitem{kondo2026mighty}
K.~Kondo, Y.~Wu, V.~Kumar, and J.~P. How, ``{MIGHTY}: {Hermite} spline-based
  efficient trajectory planning,'' \emph{{IEEE} Robotics and Automation
  Letters}, vol.~11, no.~6, pp. 6664--6671, 2026.

\bibitem{lavalle1998rapidly}
S.~LaValle, ``Rapidly-exploring random trees: A new tool for path planning,''
  \emph{Research Report 9811}, 1998.

\bibitem{kavraki2002probabilistic}
L.~E. Kavraki, P.~Svestka, J.-C. Latombe, and M.~H. Overmars, ``Probabilistic
  roadmaps for path planning in high-dimensional configuration spaces,''
  \emph{{IEEE} Transactions on Robotics and Automation}, vol.~12, no.~4, pp.
  566--580, 2002.

\bibitem{karaman2011sampling}
S.~Karaman and E.~Frazzoli, ``Sampling-based algorithms for optimal motion
  planning,'' \emph{The International Journal of Robotics Research}, vol.~30,
  no.~7, pp. 846--894, 2011.

\bibitem{richter2016polynomial}
C.~Richter, A.~Bry, and N.~Roy, ``Polynomial trajectory planning for aggressive
  quadrotor flight in dense indoor environments,'' in \emph{Robotics Research:
  The 16th International Symposium {ISRR}}.\hskip 1em plus 0.5em minus
  0.4em\relax Springer, 2016, pp. 649--666.

\bibitem{ortiz2024idb}
J.~Ortiz-Haro, W.~H{\"o}nig, V.~N. Hartmann, M.~Toussaint, and L.~Righetti,
  ``{iDb-RRT}: Sampling-based kinodynamic motion planning with motion
  primitives and trajectory optimization,'' in \emph{2024 {IEEE/RSJ}
  International Conference on Intelligent Robots and Systems ({IROS})}.\hskip
  1em plus 0.5em minus 0.4em\relax IEEE, 2024, pp. 10\,702--10\,709.

\bibitem{stoneman2014embedding}
S.~Stoneman and R.~Lampariello, ``Embedding nonlinear optimization in {RRT*}
  for optimal kinodynamic planning,'' in \emph{53rd {IEEE} Conference on
  Decision and Control}.\hskip 1em plus 0.5em minus 0.4em\relax IEEE, 2014, pp.
  3737--3744.

\bibitem{parrilo2003semidefinite}
P.~A. Parrilo, ``Semidefinite programming relaxations for semialgebraic
  problems,'' \emph{Mathematical Programming}, vol.~96, no.~2, pp. 293--320,
  2003.

\bibitem{lasserre2001global}
J.~B. Lasserre, ``Global optimization with polynomials and the problem of
  moments,'' \emph{{SIAM} Journal on Optimization}, vol.~11, no.~3, pp.
  796--817, 2001.

\bibitem{yang2020teaser}
H.~Yang, J.~Shi, and L.~Carlone, ``{TEASER}: Fast and certifiable point cloud
  registration,'' \emph{{IEEE} Transactions on Robotics}, vol.~37, no.~2, pp.
  314--333, 2020.

\bibitem{rosen2019se}
D.~M. Rosen, L.~Carlone, A.~S. Bandeira, and J.~J. Leonard, ``{SE-Sync}: A
  certifiably correct algorithm for synchronization over the special
  {Euclidean} group,'' \emph{The International Journal of Robotics Research},
  vol.~38, no. 2-3, pp. 95--125, 2019.

\bibitem{papalia2024certifiably}
A.~Papalia, A.~Fishberg, B.~W. O'Neill, J.~P. How, D.~M. Rosen, and J.~J.
  Leonard, ``Certifiably correct range-aided {SLAM},'' \emph{{IEEE}
  Transactions on Robotics}, vol.~40, pp. 4265--4283, 2024.

\bibitem{graesdal2024towards}
B.~P. Graesdal, S.~Y.~C. Chia, T.~Marcucci, S.~Morozov, A.~Amice, P.~A.
  Parrilo, and R.~Tedrake, ``Towards tight convex relaxations for contact-rich
  manipulation,'' \emph{arXiv preprint arXiv:2402.10312}, 2024.

\bibitem{kang2025global}
S.~Kang, G.~Liu, and H.~Yang, ``Global contact-rich planning with sparsity-rich
  semidefinite relaxations,'' \emph{arXiv preprint arXiv:2502.02829}, 2025.

\bibitem{shor1987quadratic}
N.~Z. Shor, ``Quadratic optimization problems,'' \emph{Soviet Journal of
  Computer and Systems Sciences}, vol.~25, pp. 1--11, 1987.

\bibitem{gillSNOPT2005}
P.~E. Gill, W.~Murray, and M.~A. Saunders, ``{SNOPT}: An {SQP} algorithm for
  large-scale constrained optimization,'' \emph{{SIAM} Review}, vol.~47, no.~1,
  pp. 99--131, 2005.

\bibitem{waechterIPOPT2006}
A.~W{\"a}chter and L.~T. Biegler, ``On the implementation of an interior-point
  filter line-search algorithm for large-scale nonlinear programming,''
  \emph{Mathematical Programming}, vol. 106, no.~1, pp. 25--57, 2006.

\bibitem{teng2023convex}
S.~Teng, A.~Jasour, R.~Vasudevan, and M.~Ghaffari, ``Convex geometric motion
  planning on {Lie} groups via moment relaxation,'' \emph{arXiv preprint
  arXiv:2305.13565}, 2023.

\bibitem{haring2024trajectory}
H.~H{\"a}ring, D.~Gramlich, C.~Ebenbauer, and C.~W. Scherer, ``Trajectory
  generation for the unicycle model using semidefinite relaxations,''
  \emph{European Journal of Control}, vol.~80, p. 101063, 2024.

\bibitem{vega2025convex}
F.~Vega, J.~Arrizabalaga, R.~Watson, and Z.~Manchester, ``Convex maneuver
  planning for spacecraft collision avoidance,'' \emph{arXiv preprint
  arXiv:2510.19058}, 2025.

\bibitem{mahajan2026tinysdp}
I.~Mahajan, J.~Arrizabalaga, A.~Grillo, F.~Vega, J.~Anderson, Z.~Manchester,
  and B.~Plancher, ``Tinysdp: Real time semidefinite optimization for
  certifiable and agile edge robotics,'' \emph{arXiv preprint
  arXiv:2605.13748}, 2026.

\bibitem{yang2025new}
L.~Yang, T.~Marcucci, P.~A. Parrilo, and R.~Tedrake, ``A new semidefinite
  relaxation for linear and piecewise-affine optimal control with time
  scaling,'' in \emph{2025 {IEEE} International Conference on Robotics and
  Automation ({ICRA})}.\hskip 1em plus 0.5em minus 0.4em\relax IEEE, 2025, pp.
  9228--9235.

\bibitem{marcucci2024shortest}
T.~Marcucci, J.~Umenberger, P.~A. Parrilo, and R.~Tedrake, ``Shortest paths in
  graphs of convex sets,'' \emph{{SIAM} Journal on Optimization}, vol.~34,
  no.~1, pp. 507--532, 2024.

\bibitem{dong2026fast}
S.~Dong, Z.~Shen, R.~Reiter, H.~Huang, B.~Gao, H.~Chen, and W.-H. Chen, ``A
  fast semidefinite convex relaxation for optimal control problems with
  spatio-temporal constraints,'' \emph{arXiv preprint arXiv:2601.03055}, 2026.

\bibitem{wei2026global}
Z.~Wei and F.~D{\"u}mbgen, ``Global sampling-based trajectory optimization for
  contact-rich manipulation via {KernelSOS},'' \emph{arXiv preprint
  arXiv:2604.27175}, 2026.

\bibitem{szeg1939orthogonal}
G.~Szeg{\H{o}}, \emph{Orthogonal polynomials}.\hskip 1em plus 0.5em minus
  0.4em\relax American Mathematical Society, 1939, vol.~23.

\bibitem{farouki1988algorithms}
R.~T. Farouki and V.~T. Rajan, ``Algorithms for polynomials in {Bernstein}
  form,'' \emph{Computer Aided Geometric Design}, vol.~5, no.~1, pp. 1--26,
  1988.

\bibitem{permenter2017reduction}
F.~N. Permenter, ``Reduction methods in semidefinite and conic optimization,''
  Ph.D. dissertation, Massachusetts Institute of Technology, 2017.

\bibitem{boydConvexOptimization}
S.~Boyd and L.~Vandenberghe, \emph{Convex Optimization}.\hskip 1em plus 0.5em
  minus 0.4em\relax Cambridge University Press, 2004.

\bibitem{stoerNumericalAnalysis}
J.~Stoer and R.~Bulirsch, \emph{Introduction to Numerical Analysis}, 3rd~ed.,
  ser. Texts in Applied Mathematics.\hskip 1em plus 0.5em minus 0.4em\relax
  Springer, 2002, vol.~12.

\bibitem{mellinger2011minimum}
D.~Mellinger and V.~Kumar, ``Minimum snap trajectory generation and control for
  quadrotors,'' in \emph{{IEEE} International Conference on Robotics and
  Automation}, 2011, pp. 2520--2525.

\bibitem{webb2013kinodynamic}
D.~J. Webb and J.~Van Den~Berg, ``Kinodynamic {RRT*}: Asymptotically optimal
  motion planning for robots with linear dynamics,'' in \emph{2013 {IEEE}
  International Conference on Robotics and Automation}.\hskip 1em plus 0.5em
  minus 0.4em\relax IEEE, 2013, pp. 5054--5061.

\bibitem{perez2012lqr}
A.~Perez, R.~Platt, G.~Konidaris, L.~Kaelbling, and T.~Lozano-Perez,
  ``{LQR-RRT*}: Optimal sampling-based motion planning with automatically
  derived extension heuristics,'' in \emph{2012 {IEEE} International Conference
  on Robotics and Automation}.\hskip 1em plus 0.5em minus 0.4em\relax IEEE,
  2012, pp. 2537--2542.

\bibitem{drake}
\BIBentryALTinterwordspacing
R.~Tedrake and the Drake Development~Team, ``Drake: Model-based design and
  verification for robotics,'' 2019. [Online]. Available:
  \url{https://drake.mit.edu}
\BIBentrySTDinterwordspacing

\bibitem{zheng2021accelerating}
D.~Zheng and P.~Tsiotras, ``Accelerating kinodynamic {RRT*} through
  dimensionality reduction,'' in \emph{2021 {IEEE/RSJ} International Conference
  on Intelligent Robots and Systems ({IROS})}.\hskip 1em plus 0.5em minus
  0.4em\relax IEEE, 2021, pp. 3674--3680.

\bibitem{werner2026biconvex}
P.~Werner, T.~Marcucci, and D.~Rus, ``Biconvex optimization for smooth
  minimum-time trajectories around convex obstacles,'' 2026, manuscript in
  preparation.

\bibitem{mattingley2012cvxgen}
J.~Mattingley and S.~Boyd, ``{CVXGEN}: A code generator for embedded convex
  optimization,'' \emph{Optimization and Engineering}, vol.~13, no.~1, pp.
  1--27, 2012.

\bibitem{gatermannParrilo2004}
K.~Gatermann and P.~A. Parrilo, ``Symmetry groups, semidefinite programs, and
  sums of squares,'' \emph{Journal of Pure and Applied Algebra}, vol. 192, no.
  1--3, pp. 95--128, 2004.

\end{thebibliography}

\end{document}